\documentclass{article}

\usepackage{microtype}
\usepackage{graphicx}
\usepackage{subfigure}
\usepackage{booktabs} 

\usepackage{hyperref}

\usepackage[accepted]{icml2020}

\icmltitlerunning{Frequency Bias in Neural Networks for Input of Non-Uniform Density}

\usepackage{amsmath}
\usepackage{amssymb}
\usepackage{amsthm}
\usepackage[colorinlistoftodos]{todonotes}
\usepackage{dsfont}
\newtheorem{theorem}[]{Theorem}
\newtheorem{lemma}[]{Lemma}
\newtheorem{proposition}[]{Proposition}

\newcommand{\Real}{\mathbb{R}}
\newcommand{\Sphere}{\mathbb{S}}

\newcommand{\norm}[1]{\left\lVert#1\right\rVert}
\newcommand{\abs}[1]{\left\vert#1\right\rvert}

\newcommand{\aw}{\mathbf{a}}
\newcommand{\bias}{\mathbf{b}}

\newcommand{\h}{\mathbf{h}}
\newcommand{\uu}{\mathbf{u}}
\newcommand{\vv}{\mathbf{v}}
\newcommand{\w}{\mathbf{w}}
\newcommand{\x}{\mathbf{x}}
\newcommand{\y}{\mathbf{y}}

\newcommand{\rb}[1]{\textcolor{blue}{[Ronen: #1]}}

\definecolor{mypink1}{rgb}{1, 0.6, 0}

\begin{document}

\twocolumn[
\icmltitle{Frequency Bias in Neural Networks for Input of Non-Uniform Density}




\begin{icmlauthorlist}
\icmlauthor{Ronen Basri}{wis}
\icmlauthor{Meirav Galun}{wis}
\icmlauthor{Amnon Geifman}{wis}
\icmlauthor{David Jacobs}{umd}
\icmlauthor{Yoni Kasten}{wis}
\icmlauthor{Shira Kritchman}{wis}
\end{icmlauthorlist}

\icmlaffiliation{wis}{Department of Computer Science and Applied Mathematics, Weizmann Institute of Science, Rehovot, Israel}
\icmlaffiliation{umd}{Department of Computer Science, Univeristy of Maryland, College Park, MD, USA}

\icmlcorrespondingauthor{Ronen Basri}{ronen.basri@weizmann.ac.il}

\icmlkeywords{Machine Learning, ICML}

\vskip 0.3in
]



\printAffiliationsAndNotice{}  

\begin{abstract}
Recent works have partly attributed the generalization ability of over-parameterized neural networks to frequency bias -- networks trained with gradient descent on data drawn from a uniform distribution find a low frequency fit before high frequency ones. As realistic training sets are not drawn from a uniform distribution, we here use the Neural Tangent Kernel (NTK) model to explore the effect of variable density on training dynamics. Our results, which combine analytic and empirical observations, show that when learning a pure harmonic function of frequency $\kappa$, convergence at a point $\x \in \Sphere^{d-1}$ occurs in time $O(\kappa^d/p(\x))$ where $p(\x)$ denotes the local density at $\x$. Specifically, for data in $\Sphere^1$ we analytically derive the eigenfunctions of the kernel associated with the NTK for two-layer networks. We further prove convergence results for deep, fully connected networks with respect to the spectral decomposition of the NTK. Our empirical study highlights similarities and differences between deep and shallow networks in this model.
\end{abstract}

\section{Introduction}

A key question in understanding the success of neural networks is: what makes over-parameterized networks generalize so well, avoiding solutions that overfit the training data? In search of an explanation, a number of recent papers \cite{farnia2018spectral,rahaman2019spectral,Xu2019} have suggested that training with gradient descent (GD) (as well as SGD) yields a frequency bias -- in early epochs training a neural net yields a low frequency fit to the target function, while high frequencies are learned only in later epochs, if they are needed to fit the data (see Figure \ref{fig:motivation}(top)). 

\begin{figure}[t]
    \centering
    \includegraphics[width=3.25cm]{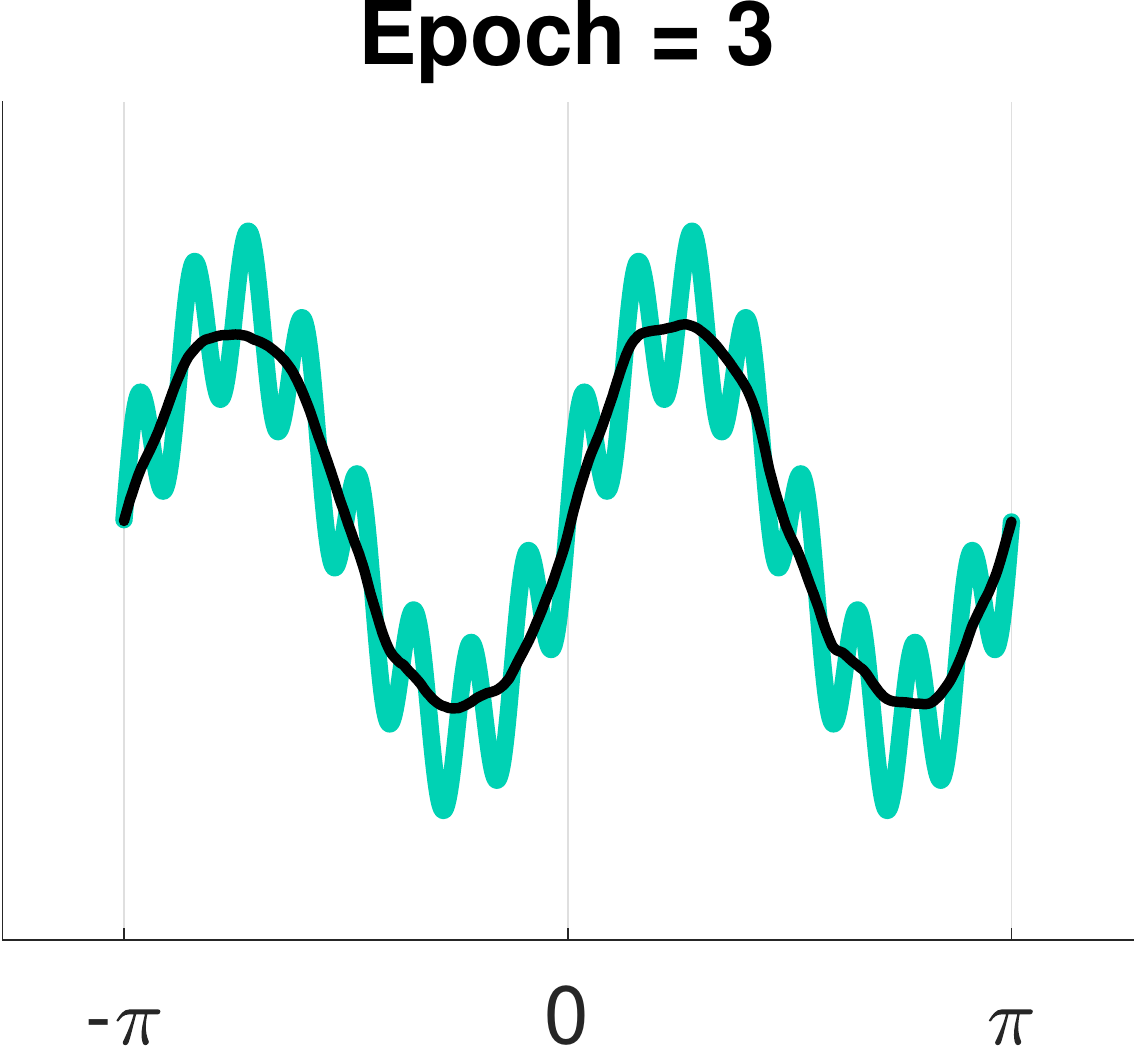}
    \includegraphics[width=3.25cm]{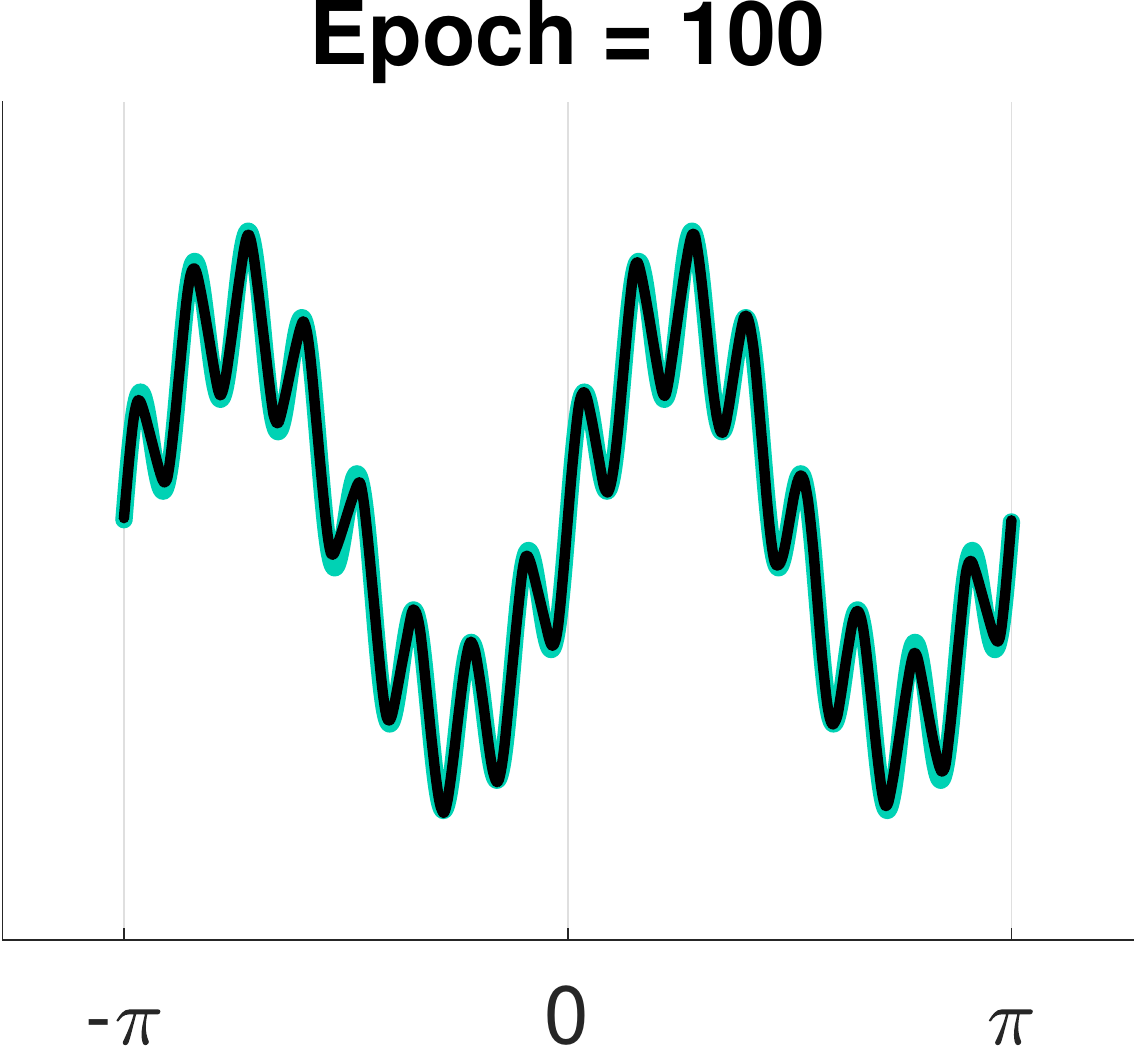}\\[0.2cm]
    \includegraphics[width=3.25cm]{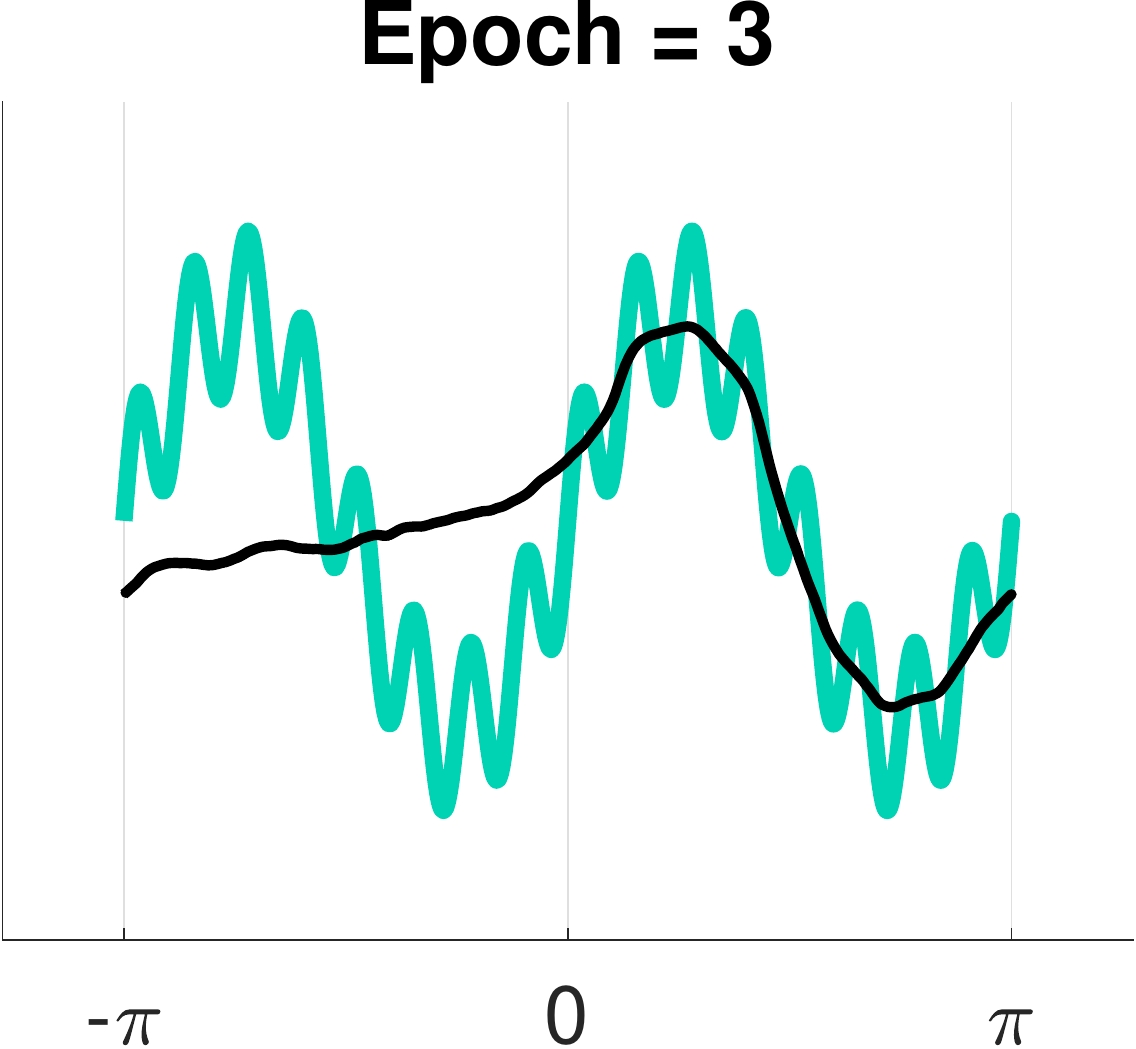}
    \includegraphics[width=3.25cm]{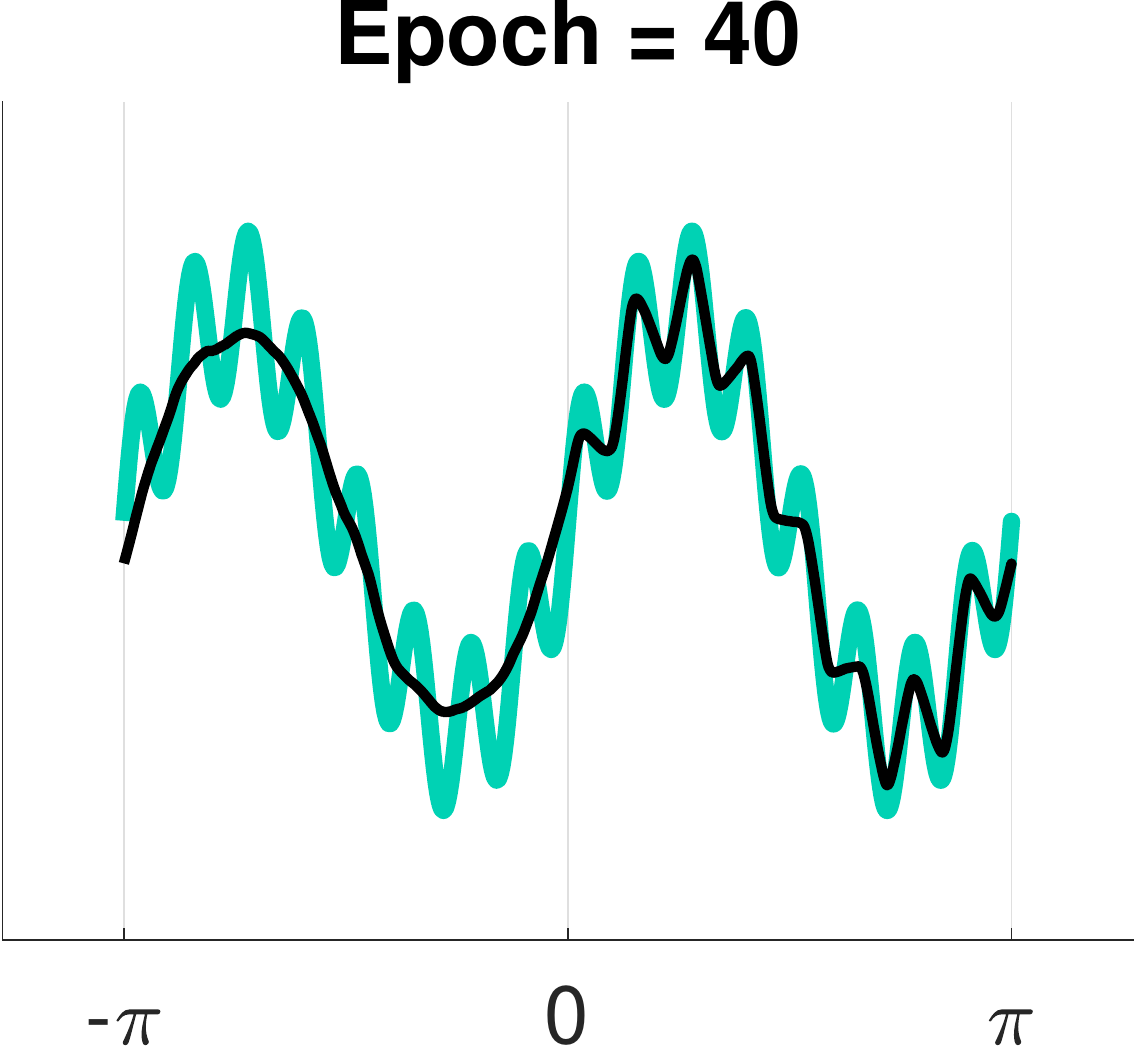}
    \caption{Frequency bias under uniform (top) and non-uniform (bottom) distributions. The light cyan line represents the target function which is composed of the sum of a low and high frequency functions. The thin black line represents the network output. Top: when training data is distributed uniformly, low frequency (left) is learned before high frequency (right). Bottom: with non-uniform distribution (positive region is dense, negative is sparse), a good low frequency fit for the low density region is obtained only after 40 epochs, but by then the network fits most of the high frequency component of the target function at the dense region.}
    \label{fig:motivation}
\end{figure}

This frequency bias has been carefully analyzed in the case of over-parameterized, two-layer networks with Rectified Linear Unit (ReLU) activation, when only the first layer is trained. The dynamics of GD in this case was shown to match the dynamics of GD for the corresponding Neural Tangent Kernel (NTK) \cite{arora2019fine,du2018gradient,Jacot2018}. Assuming the training data is distributed uniformly on a hypersphere, the NTK matrix forms a convolution on the sphere. Its eigenvectors consist of the spherical harmonic functions \cite{Basri2019,xie2017diverse}, and its eigenvalues shrink monotonically with frequency, yielding longer convergence times for high frequency components. Specifically, for training data on the circle, high frequencies are learned quadratically slower than low frequencies, and this frequency-dependent gap increases exponentially with dimension \cite{Basri2019,bietti2019inductive,cao2019}.

All this previous work assumed that training data is distributed uniformly. However, realistic training datasets are distributed with a non-uniform density. A natural question therefore is to what extent frequency bias is exhibited for such datasets? Below we provide evidence that frequency bias interacts with density. We show that in any region of the input space with locally constant density, low frequencies are still learned much faster than high frequencies, but the rate of learning also depends linearly on the density. This phenomenon is demonstrated in Figure \ref{fig:motivation}(bottom).

Our paper contains both theoretical and empirical results. We first focus on analyzing the NTK model for two-layer networks with ReLU activation and 2D input, normalized to lie on the unit circle, allowing for input drawn from a non-uniform density that is piecewise constant.  For this model we derive closed form expressions for its eigenfunctions and eigenvalues. These eigenfunctions contain functions of piecewise constant local frequency, with higher frequencies where the density of the training data is higher. This implies that we learn high frequency components of a target function faster in regions of higher density.  This also allows us to prove that a pure 1-dimensional sine function of frequency $\kappa$ is learned in time $O(\kappa^2/p^*)$, where $p^*$ denotes the minimum density in the input space. Our experiments illustrate these results and further suggest that for input on a $d-1$-dimensional hypersphere, spherical harmonics are learned in time $O(\kappa^d/p^*)$.  

We next examine the NTK for deep, fully connected (FC) networks. We first prove that given a target function $y(\x)$, training networks of finite width with GD converges to $y$ at a speed that depends on the projection of $y$ over the eigenvectors of the NTK, extending previous results proved for two-layer networks \cite{arora2019fine,cao2019}. We further show that for uniform data the eigenfunctions of NTK consist of the spherical harmonics. We complement these observations with several empirical findings. (1) We show that for uniformly distributed data the eigenvalues decay with frequency, suggesting that frequency bias exists also in deep FC networks. Moreover, similar to two-layer networks, a pure harmonic function of frequency $\kappa$ is learned in time $O(\kappa^d)$ asymptotically in $\kappa$. However, deeper networks appear to learn frequencies of lower $k$ faster than shallow ones. (2) For training data drawn from non-uniform densities the eigenfunctions of NTK appear indistinguishable from those obtained for two-layer networks, indicating that with deep nets learning a harmonic of frequency $\kappa$ should also require $O(\kappa^d/p^*)$ iterations.

Our results have several implications.  First, we extend results that have been proven for training data with a uniform density to the more realistic case of non-uniform density, also extending results for shallow networks to deep, fully connected networks.  These results support the idea that real neural networks have a frequency bias that can explain their ability to avoid overfitting.  Second, while it is not surprising that networks fit functions of all frequencies more slowly in regions with low data density, we demonstrate that this is the case and quantify this effect.  Our results have an interesting implication for training that uses early stopping to regularize the solution.  Suppose the signal one wishes to fit is low frequency, and it is corrupted by high frequency noise.  Because a network learns low frequency signals more slowly in regions of low density, by the time the signal is learned in these regions, the network will also have learned high frequency components of the noise in regions of high density. This is illustrated in Figure~\ref{fig:motivation}(bottom). 

\comment{
-------
Figures:
\begin{enumerate}
    \item Motivating example: Target function made of low frequency + high frequency noise (of small amplitude). Data is composed, say, of two constant densities. Show fit at different epochs. Essentially we want to show that network cannot fit the low frequency in all space at any given time. i.e. to get the low frequency at the sparse part you already fit the noise in the dense part \rb{Should we show in contrast that this is possible with uniform density? Maybe not..}
    \item 2 layers, NTK: the eigenfunctions obtained for a certain density. Also for 2D data.
    \item 2 layers, NTK: a plot of the eigenvalues + normalized by Z to show they all coincide.. Also for 2D data.
    \item 2 layers network: convergence times for several densities + normalized so that all parabolas coincide
    \item Inverse linear relation of convergence time and density. Also for 2D data.
    \item Deep NTK: eigenfunctions for uniform and non-uniform densities. (Maybe unnecessary because they are the same as the shallow case, or maybe plot just the ones for a deep net?)
    \item Deep, NTK: exponent as a function of number of layers for 1D, 2D and maybe 3D.
    \item Deep network: actual convergence times for 3, 5 and 7 layers. Show a fit to the previous figure. Also for 2D input.
\end{enumerate}
FC deep nets (NTK):
\begin{enumerate}
    \item Proof of convergence with fine grained.
    \item Eigen functions under the uniform density are spherical harmonics and are ordered according to frequency.
    \item Complexity (exponent) of convergence rate under the uniform density as a function of depth and dimension (graph).
    \item Run experiments to fit this graph (using Yoni's code).
    \item Is there degeneracy with very deep FC?
    \item Non-uniform for deep FC + experiment. Also Resnet?
    \item Verify theory for non-uniform. Are the eigenvalues identical to the uniform case?
    \item Probably not in this round: NTK for Resnet
\end{enumerate}
}

\section{Prior work}

Many recent papers attempt to explain the generalization ability of overparameterized nets. Perhaps the most convincing relate overparameterized networks to kernel methods. \cite{Jacot2018} identified a family of kernels, termed Neural Tangent Kernels, and showed that neural networks behave like these kernels, in the limit of infinite widths.
Related work investigated variants of these kernels, showing that networks of finite, albeit very large widths converge to zero training error almost always and deriving generalization bounds for such networks. These analyses were applied to two-layer networks \cite{bach2017breaking,bietti2019inductive,du2018gradient,Vempala2018,xie2017diverse}, multilayer perceptrons (i,e. fully connected), residual and convolutional networks  \cite{allen2018convergence,allenzhu,arora,huang2019dynamics,lee2019wide}. 

However, these kernel models have been criticised for requiring unrealistically wide networks. Additionally, it is still debated if such linear dynamics (referred to as ``lazy training'') fully explain the performance of neural networks. Recent theoretical and empirical results suggest that NTK models still somewhat underperform common nonlinear networks \cite{arora,Chizat2019,Novak2019,Woodworth2019}.

Other work suggested that networks are biased to learn simple functions, and in particular that GD proceeds by first fitting a low frequency function to the target function, and only fits the higher frequencies in later epochs
\cite{rahaman2019spectral,Xu2019,farnia2018spectral}. Additional work \cite{bach2017breaking,Basri2019,bietti2019inductive,cao2019} proved the existence of frequency bias in NTK models of two-layer networks and derived convergence rates of training as a function of target frequency. All of these works assumed that training data is distributed uniformly. \cite{Canu1999} proposed loss functions that allow higher frequency fit in regions where training data is dense, and only low frequency fit in the sparse regions. Our results suggest that such a penalization may be implicitly enforced in NTK models.

Classical work on kernel methods acknowledged the importance of understanding the eigenfunctions and eigenvalues of kernels for non-uniform data distributions, but focused mainly on bounding the difference between the empirical kernel matrix and the theoretical kernel for the given distribution (e.g., \cite{ShaweTaylor2005,Williams2000}). \cite{Liang2013} derived analytic expressions for the eigenfunctions of polynomial kernels. \cite{Goel2017} investigated the gram matrix of the data distribution and showed that sufficiently fast decay of its eigenvalues allows learnability by neural networks. We are unaware of works that derive analytic expressions for the eigenfunctions of NTK under non-uniform distributions.

\comment{Left outs: 
Papers showing implicit regularization (https://arxiv.org/abs/1810.01075).\\
Convergence results for two-layers (like Arora fine grained) + spherical harmonics for uniform distribution
(https://arxiv.org/pdf/1912.01198.pdf)\\
The dynamics of gradient descent bias the model towards simple solutions - initialization separates deep and shallow
(https://arxiv.org/pdf/1909.12051.pdf)
}

\section{Preliminaries}

We consider in this work NTK models for fully connected neural networks with rectified linear unit (ReLU) activations. These kernels are defined through the following formula
\begin{equation}\label{eq:kernel_main} 
    k(\x_i,\x_j) = \mathbb{E}_{\w \sim {\cal I}} \left< \frac{\partial f(\x_i,\w)}{\partial \w}, \frac{\partial f(\x_j,\w)}{\partial \w} \right>,
\end{equation}
where $f(\x,\w)$ is the network output for point $\x \in \Real^{d}$ with parameters $\w$, $\x_i$ and $\x_j$ are any two training points, and the expectation is over the possible initializations of $\w$, denoted ${\cal I}$ (usually normal distribution).

 We first consider a two layer network with bias:
\begin{equation}  \label{eq:network}
    f(\x; \w) = \frac{1}{\sqrt{m}} \sum_{r=1}^m a_r\sigma(\w_r^T \x + b_r),
\end{equation}
where $\|\x\|=1$ (denoted $\x \in \Sphere^{d-1}$) is the input, the vector $\w$ includes the weights and bias terms of the first layer, denoted respectively $W=[\w_1,...,\w_m] \in \Real^{d \times m}$ and $\bias=[b_1,...,b_m]^T \in \Real^m$, as well as the weights of the second layer, denoted $\aw=[a_1,...,a_m]^T \in \Real^m$. $\sigma$ denotes the ReLU function, $\sigma(x)=\max(x,0)$. Bias is important in the case of two-layer networks since \cite{Basri2019} without bias such networks are non-universal and cannot express harmonic functions of odd frequencies except frequency 1.

We then consider deep fully-connected networks with $L+1>2$ layers. For such networks we forgo the bias since our empirical results (Section \ref{sec:deep}) indicate that they are universal even without bias. These networks are expressed as 
\begin{align}\label{model:arora}
    & f(\x;\w) = W^{(L+1)} \cdot \sqrt{\frac{c_{\sigma}}{d_L}} \sigma \left(W^{(L)} \cdot \right. \nonumber \\ & \left. \sqrt{\frac{c_{\sigma}}{d_{L-1}}} \sigma \left( W^{(L-1)} \cdots \sqrt{\frac{c_{\sigma}}{d_{1}}} \sigma \left( W^{(1)} \x \right)\right) \right),
\end{align}
where $\x \in \Real^{d_1}$, $\|\x\|=1$, the parameters $\w$ include $W^{(L+1)},W^{(L)},...,W^{(1)}$, where $W^{(l)} \in \Real^{d_l \times d_{l-1}}$, $W^{(L+1)} \in \Real^{1 \times d_L}$, and $c_{\sigma} = 1/\left( \mathbb{E}_{z \sim \mathcal{N}(0,1)} [\sigma(z)^2] \right) = 2$.

We assume that $n$ training points are sampled i.i.d.\ from an arbitrary distribution $p(\x)$ on the hypersphere and that each sample $\x_i$ is supplied with a target value $y_i \in \Real$ from an unknown function $y_i=g(\x_i)$. Our theoretical derivations further assume that $p(\x)$ is piecewise constant. The network is trained to minimize the $\ell_2$ loss 
\begin{equation}
\Phi(\w)=\frac{1}{2} \sum_{i=1}^n (y_i - f(\x_i;\w))^2.    
\end{equation}
using gradient descent (GD). 

For our analysis, to simplify the NTK expressions, in the case of a two-layer network we only train the weights and bias of the first layer (as in \cite{arora2019fine,du2018gradient}). We initialize these weights from a normal distribution $\w_r^{(0)}, b_r^{(0)} \sim {\cal N}(0,\tau^2 I)$. We further initialize $a_r$ from a uniform distribution on $\{-1,1\}$ and keep those weights fixed. In the case of deep networks we train all the weights, initializing by $\w \sim {\cal N}(0,I)$.

We next provide expressions for the corresponding neural tangent kernels. For a two-layer network with bias where only the first layer weights are trained the corresponding NTK takes the form \cite{Basri2019}
\begin{equation}  \label{eq:kernel2layer}
    k(\x_i,\x_j) = \frac{1}{4\pi}(\x_i^T\x_j+1)(\pi - \arccos(\x_i^T\x_j)).
\end{equation}
When the training data is distributed uniformly, this kernel forms a convolution operator, and so its eigenfunctions are the spherical harmonics on the hypersphere $\Sphere^{d-1}$ (or Fourier series when $d=2$). The eigenvalues shrink at the rate of $O(1/\kappa^{d})$, where $\kappa$ denotes the frequency of the spherical harmonic functions. Gradient descent training of a target function composed of a pure harmonic requires a number of iterations that is inversely proportional to the corresponding eigenvalue, i.e., $O(\kappa^{d})$. \cite{bach2017breaking,Basri2019,bietti2019inductive,cao2019,xie2017diverse}

For a deep FC network the NTK is expressed by the following recursion \cite{arora, Jacot2018}
\begin{equation}  \label{eq:ntkdeep}
    \Theta_{\infty}^{(L)}(\x_i,\x_j) = \Theta_{\infty}^{(L-1)}(\x_i,\x_j) \dot{\Sigma}^{(L)}(\x_i,\x_j) + \Sigma^{(L)}(\x_i,\x_j),
\end{equation}
where for $h \in [L]$
\begin{align*}
\Sigma^{(0)}(\x_i,\x_j) &= \x_i^T \x_j \\ 
\Lambda^{(h)}(\x_i,\x_j) &= \begin{bmatrix} \Sigma^{(h-1)}(\x_i,\x_i) & \Sigma^{(h-1)}(\x_i,\x_j) \\
\Sigma^{(h-1)}(\x_j,\x_i) & \Sigma^{(h-1)}(\x_j,\x_j) 
\end{bmatrix}\\
\Sigma^{(h)}(\x_i,\x_j) &= c_{\sigma} \mathbb{E}_{(u,v) \sim \mathcal{N}(0,\Lambda^{(h)})} [\sigma(u) \sigma(v)] \\
\dot{\Sigma}^{(h)}(\x_i,\x_j) &= c_{\sigma} \mathbb{E}_{(u,v) \sim \mathcal{N}(0,\Lambda^{(h)})} [\dot{\sigma}(u) \dot{\sigma}(v)].
\end{align*}
Here $\dot{\sigma}(\cdot)$ denotes the step function (i.e., the derivative of the ReLU function).
The covariance matrices have the form $\Lambda = \begin{bmatrix} 1 & \rho \\ \rho & 1 \end{bmatrix}$ with $|\rho| \leq 1$, and the expectations have the following closed form expressions
\begin{align*}
&\mathbb{E}_{(u,v) \sim \mathcal{N}(0,\Lambda^{(h)})} [\sigma(u) \sigma(v)]=\frac{\rho(\pi-\arccos(\rho))+\sqrt{1-\rho^2} }{2 \pi} \\
&\mathbb{E}_{(u,v) \sim \mathcal{N}(0,\Lambda^{(h)})} [\dot{\sigma}(u) \dot{\sigma}(v)] = \frac{\rho(\pi-\arccos(\rho))}{2 \pi}.
\end{align*}

\section{The eigenfunctions of NTK for two-layer networks for non-uniform distributions}

We begin by investigating the NTK model for two-layer networks when the training is drawn from a non-uniform distribution. Focusing first on 1D target functions $y(\x):\Sphere^1 \rightarrow\Real$ and a piecewise constant data distribution $p(\x)$, we derive explicit expressions for the eigenfunctions and eigenvalues of NTK. This allows us to prove that learning a one-dimensional function of frequency $\kappa$ requires $O(\kappa^2/p^*)$ iterations, where $p^*$ denotes the minimal density in $p(x)$. We complement these theoretical derivations with experiments with functions in higher dimensions, which indicate that learning functions of frequency $\kappa$ in $\Sphere^{d-1}$ requires 
$O(\kappa^d/p^*)$ iterations.

Consider the NTK model described in  \eqref{eq:kernel2layer}, which corresponds to an infinitely wide, two-layer network for which only the first layer is trained. Suppose that $n$ training data points are sampled from a non-uniform, piecewise constant distribution $p(\x)$ on the circle, $\x \in \Sphere^1$. We then form an $n \times n$ matrix $H^p$ whose entries for samples $\x_i$ and $\x_j$ consist of $H^p_{ij}=k(\x_i,\x_j)$, with $k$ as defined in \eqref{eq:kernel2layer}. Following \cite{arora2019fine}, the convergence rates of GD for such a network will depend on the eigen-system of $H^p$. To analyze this eigen-system, we consider the limit of $H^p$ as the number of points goes to infinity. In this limit the eigen-system of $H^p$ approaches the eigen-system of the kernel $k(\x_i,\x_j)p(\x_j)$, where the eigenfunctions $f(x)$ satisfy the following equation \cite{ShaweTaylor2005,Williams2000},
\begin{equation}  \label{eq:eigfun}
    \int_{\Sphere^1} k(\x_i,\x_j)p(\x_j)f(\x_j)d\x_j = \lambda f(\x_i).
\end{equation}
This is a homogeneous Fredholm Equation of the second kind with the non-symmetric polar kernel $k(\x_i,\x_j)p(\x_j)$. The existence of the eigenfunctions with real eigenvalues is established by symmetrizing the kernel. Let $\tilde k(\x_i,\x_j) = p^{1/2}(\x_i)k(\x_i,\x_j)p^{1/2}(\x_j)$ and $g(\x)=p^{1/2}(\x)f(\x)$. Multiplying \eqref{eq:eigfun} by $p^{1/2}(\x_i)$ yields
\begin{equation}
    \int_{\Sphere^d} \tilde k(\x_i,\x_j)g(\x_j)d\x_j = \lambda g(\x_i),
\end{equation}
implying the eigenfunctions exist and $\lambda$ is real.

We next parameterize the unit circle by angles, and denote by $x,z$ any two angles. We can therefore express \eqref{eq:eigfun} as
\begin{equation}  \label{eq:eigf_1d}
    \int_{x-\pi}^{x+\pi} k(x,z)p(z)f(z)dz = \lambda f(x),
\end{equation}
where the kernel in \eqref{eq:kernel2layer} expressed in terms of angles reads
\begin{equation}  \label{eq:kernel_1d}
    k(x,z) = \frac{1}{4\pi} (\cos(x-z)+1)(\pi-|x-z|).
\end{equation}
Both $p(x)$ and $f(x)$ are periodic with a period of $2\pi$ since $x$ lies on the unit circle.

\subsection{Explicit expressions for the eigenfunctions}

Below we solve \eqref{eq:eigf_1d} and derive an explicit expression for the eigenfunctions $f(x)$. Our derivation assumes that $p(x)$ is piecewise constant. While this assumption limits the scope of our solution, empirical results suggest that when $p(x)$ changes continuously the eigenfunctions are modulated continuously, consistently with our solution. We summarize:
\begin{proposition}
Let $p(x)$ be a piecewise constant density function on $\Sphere^1$. Then the eigenfunctions in \eqref{eq:eigf_1d} take the general form
\begin{equation}  \label{eq:solution}
    f(x) = a(p(x)) \cos\left(\frac{q}{Z}\Psi(x)+b(p(x))\right),
\end{equation}
where $q$ is integer, $\Psi(x)=\int_{-\pi}^x \sqrt{p(\tilde x)}d\tilde x$ and $Z=\frac{1}{2\pi}\Psi(\pi)$.
\end{proposition}

Note that if $p(x)=p_j$ is constant in a connected region $R_j \subseteq \Sphere^1$, then \eqref{eq:solution} can be written as
\begin{equation}  \label{eq:solutionpc}
    f(x)=a_j \cos\left(\frac{q\sqrt{p_j}x}{Z}+b_j\right), \forall x \in R_j.
\end{equation}
In other words, over the region $R_j$, this is a cosine function with frequency proportional to $\sqrt{p_j}$. A plot of eigenfunctions for a piecewise constant distribution is shown in Fig.~\ref{fig:efun1d}.

\begin{figure}[tb]
    \centering
    \includegraphics[width=\linewidth]{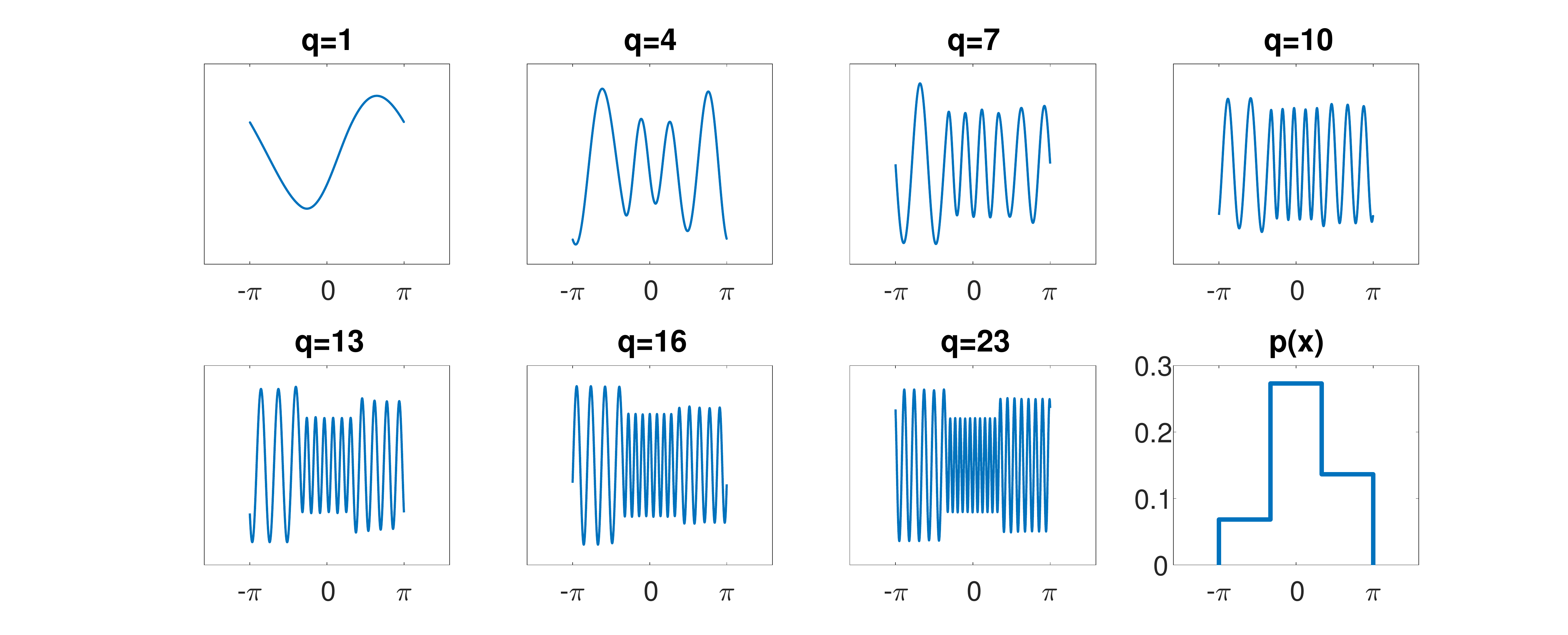}
    \caption{For the NTK of a two-layer network with bias we plot its eigenfunctions (in a decreasing order of eigenvalues) under a non-uniform data distribution in $\Sphere^1$. Here we used a density composed of three constant regions with $p(x) \in 3/(2\pi)\{1/7,2/7,4/7\}$ (bottom right plot).}
    \label{fig:efun1d}
%
    \centering
    \includegraphics[width=4.5cm]{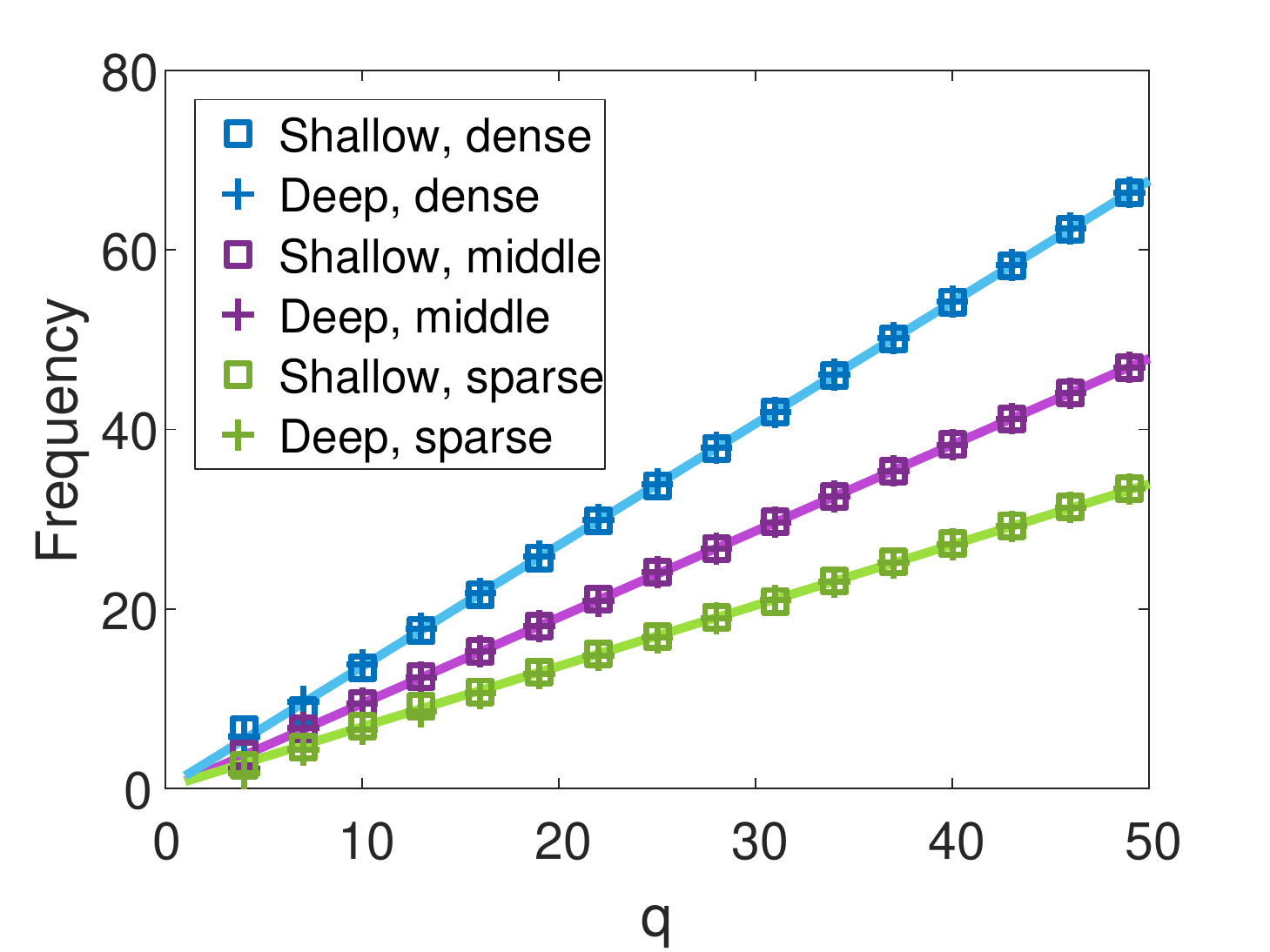}
    \caption{The local frequency in the eigenfunctions within each of the three constant region densities in Figure~\ref{fig:efun1d}, plotted for both a two-layer and deep (depth=10) networks (marked respectively by squares and plus signs). Measurements are obtained by applying FFT to each region. The measurements are in close match to our formula \eqref{eq:solutionpc} (solid line).}
    \label{fig:freq1d}
\end{figure}

The proof of the proposition relies on a lemma, proved in supplementary material, stating that the solution to \eqref{eq:eigf_1d} satisfies the following second order ordinary differential equation (ODE) \begin{equation}  \label{eq:2nd_order}
    f''(x)=-\frac{p(x)}{\pi\lambda}f(x).
\end{equation}
In a nutshell, the lemma proved by applying a sequence of six derivatives to \eqref{eq:eigf_1d} with respect to $x$, along with some algebraic manipulations, yielding a sixth order ODE for $f(x)$. Assuming that $p(x)$ is piecewise constant simplifies the ODE. Then \eqref{eq:2nd_order} is obtained by restricting $p(x)$ to have a period of $\pi$, but this restriction can be lifted by preprocessing the data in a straightforward way without changing the function that needs to be learned.

Eq.~\eqref{eq:2nd_order} has the following general solutions
\begin{equation}
    f(x) = Ae^{i\frac{\Psi(x)}{\sqrt{\pi\lambda}}x} + Be^{-i\frac{\Psi(x)}{\sqrt{\pi\lambda}}x},
\end{equation}
such that the derivative of $\Psi$ is $\Psi'(x)=\sqrt{p(x)}$, resulting in real eigenfunctions of the form
\begin{equation}
    f(x) = a(p(x)) \cos\left(\frac{\Psi(x)}{\sqrt{\pi\lambda}}x + b(p(x))\right).
\end{equation}
As with the uniform distribution, due to periodic boundary conditions there is a countable number of eigenvalues, and those can be determined (up to scale) using the known eigenvalues for the uniform case \cite{Basri2019}. With this we obtain
\begin{eqnarray}  \label{eq:ck1}
    \lambda &=&
\left\{ \begin{array}{ll}
     Z^2 \left( \frac{1}{2\pi^2} + \frac{1}{8} \right) & q=0\\[0.06cm]
     Z^2 \left( \frac{1}{\pi^2} + \frac{1}{8} \right) & q=1\\[0.06cm]
     \frac{Z^2(q^2+1)}{\pi^2 (q^2-1)^2}& q \ge 2 ~~\rm{even}  \\[0.06cm]
     \frac{Z^2}{\pi^2 q^2} & q \ge 2 ~~\rm{odd.} 
\end{array}\right.
\end{eqnarray}
$q$ is integer, and there is one eigenfunction for $q=0$ and two eigenfunctions for every $q>0$. Figure \ref{fig:ev} shows a plot of the eigenvalues computed for various densities.

\begin{figure}[tb]
    \centering
    \includegraphics[width=3.5cm]{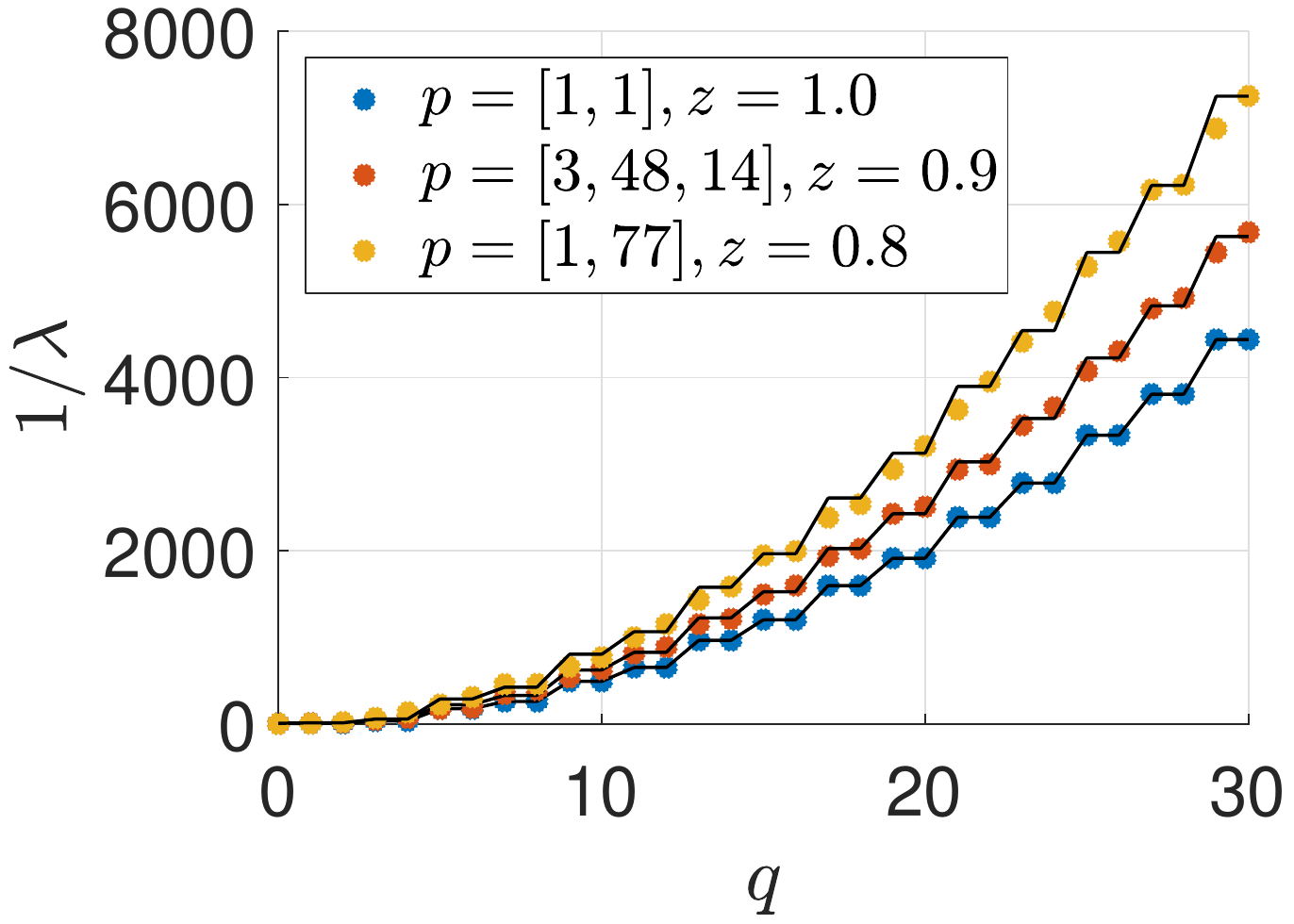}
    \caption{The kernel eigenvalues for several distributions. The formula (marked by the solid lines) closely matches the eigenvalues $H^p$ computed numerically using $50K$ points.}
    \label{fig:ev}
\end{figure}

\begin{figure}[tb]
    \centering
    \includegraphics[width=\linewidth]{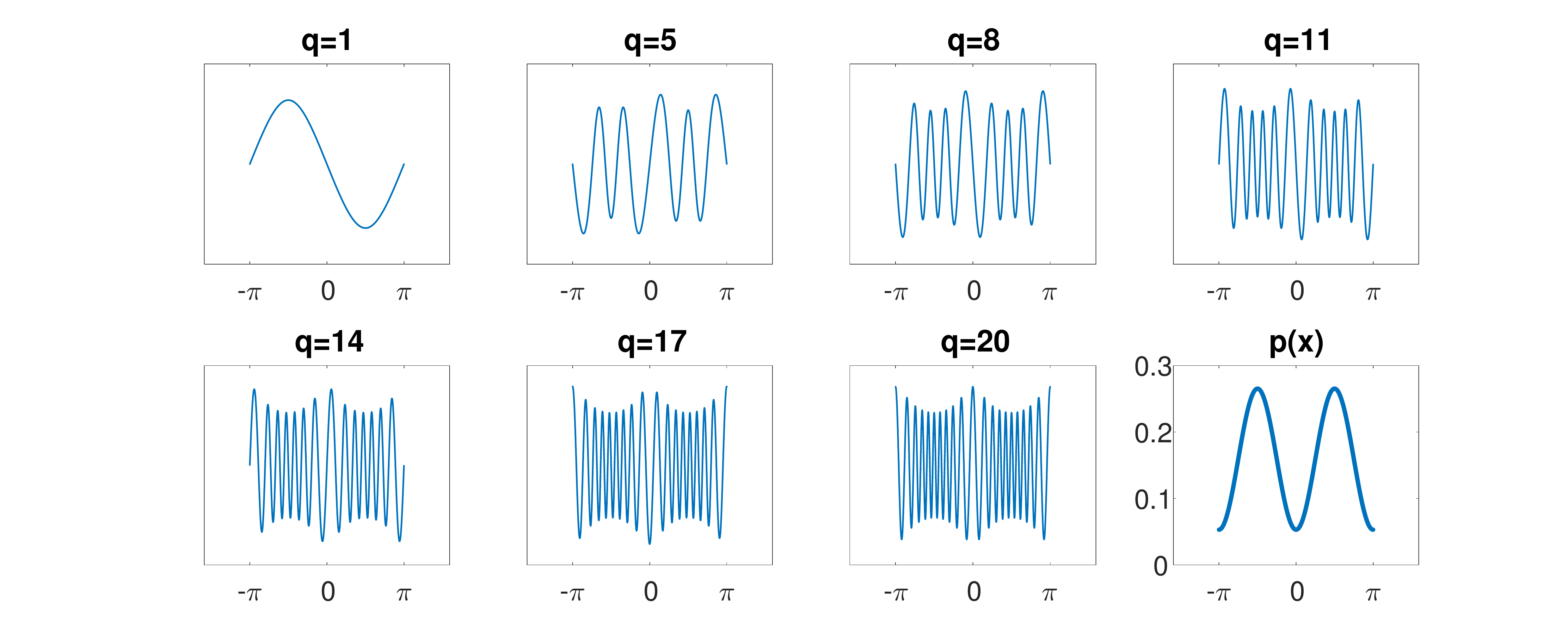} 
    \caption{For the NTK of a two-layer network we plot the eigenvectors of $H^p$ for a continuous distribution, $p(x)=\frac{3\cos(2x+\pi)+4.5}{9\pi}$ (bottom right).}
    \label{fig:efun1dcos}
\end{figure}

The amplitudes and phase shifts are determined by requiring the eigenfunctions to be continuous and differentiable everywhere. We show in supplementary material that for two neighboring regions, $j,j+1$ it holds that if $p_j\leq p_{j+1}$ then the ratio of the amplitudes is bounded (tightly) for different values of $p_j$ and $p_{j+1}$ as follows:
\begin{equation}
    1\leq \frac{a_j}{a_{j+1}} \leq \sqrt{\frac{p_{j+1}}{p_j}}.
\end{equation}

Figure~\ref{fig:efun1d} shows the eigenvectors and eigenvalues for an example of a piecewise constant distribution. It can be seen that each eigenfunction consists of a piecewise sine function; i.e., the eigenfunctions in every region where $p(x)$ is constant form pure sine functions with frequency that changes from one region to the next. As we inspect eigenfunctions with decreasing eigenvalues we find, as our theory shows (see Figure~\ref{fig:freq1d}), that the frequencies increase in all regions, but for all eigenfunctions they maintain constant ratios that are equal to the ratios between the square roots of the corresponding densities. Finally, Figure \ref{fig:efun1dcos} shows the eigenvectors of $H^p$ for a continuous distribution, showing similar behaviour to our analytic expressions.

\subsection{Time to convergence}
Determining the eigenfunctions and eigenvalues of the NTK allows us to predict the number of iterations needed to learn target functions and to understand effects due to varying densities. To understand this we consider target functions of the form $g(x)=\cos(\kappa x)$ where $x$ is drawn from a piecewise constant distribution $p(x)$ on $\Sphere^1$. Denote by $R_j \subseteq \Sphere^1$, $1 \le j \le l$ the regions of constant density. Loosely speaking (see Figure \ref{fig:illustration}), for each region $R_j$ we expect $g(x)$ to correlate well with one eigenfunction (and perhaps to additional ones, but with less energy). Of these, the region corresponding to the lowest density should correlate with an eigenfunction with the smallest eigenvalue. This eigenvalue, which depends on both the target frequency $\kappa$ and the density $p(x)$ within that region, will determine the number of iterations to convergence. This is summarized in the following theorem.
\begin{theorem}  \label{thm:convergence}
Let $p(x)$ be a piecewise constant distribution on $\Sphere^1$. Denote by $u^{(t)}(x)$ the prediction of the network at iteration $t$ of GD. For any $\delta>0$ the number of iterations $t$ needed to achieve $\|g(x)-u^{(t)}(x)\|<\delta$ is $\tilde O(\kappa^2/p^*)$, where $p^*$ denotes the minimal density of $p(x)$ in $\Sphere^1$ and $\tilde O(.)$ hides logarithmic terms.
\end{theorem}
Proving this theorem is complicated by the fact that (1) the frequency of the target function may not be exactly represented in the eigenfunctions of the kernel, due to the discrete number of eigenfunctions, and (2) the eigenfunctions restricted to any given region $R_j$ are not orthogonal. These two properties may result in non-negligible correlations of $g(x)$ with eigenfunctions of yet smaller eigenvalues. Therefore, to prove Theorem \ref{thm:convergence} we first inspect the projections of $g(x)$ onto the eigenfunctions corresponding to such small eigenvalues and prove a bound on this tail. Subsequently we use this bound to prove the convergence rate in the theorem. The proofs are provided in the supplementary material.

\begin{figure}[tb]
    \centering
    \includegraphics[height=2.6cm]{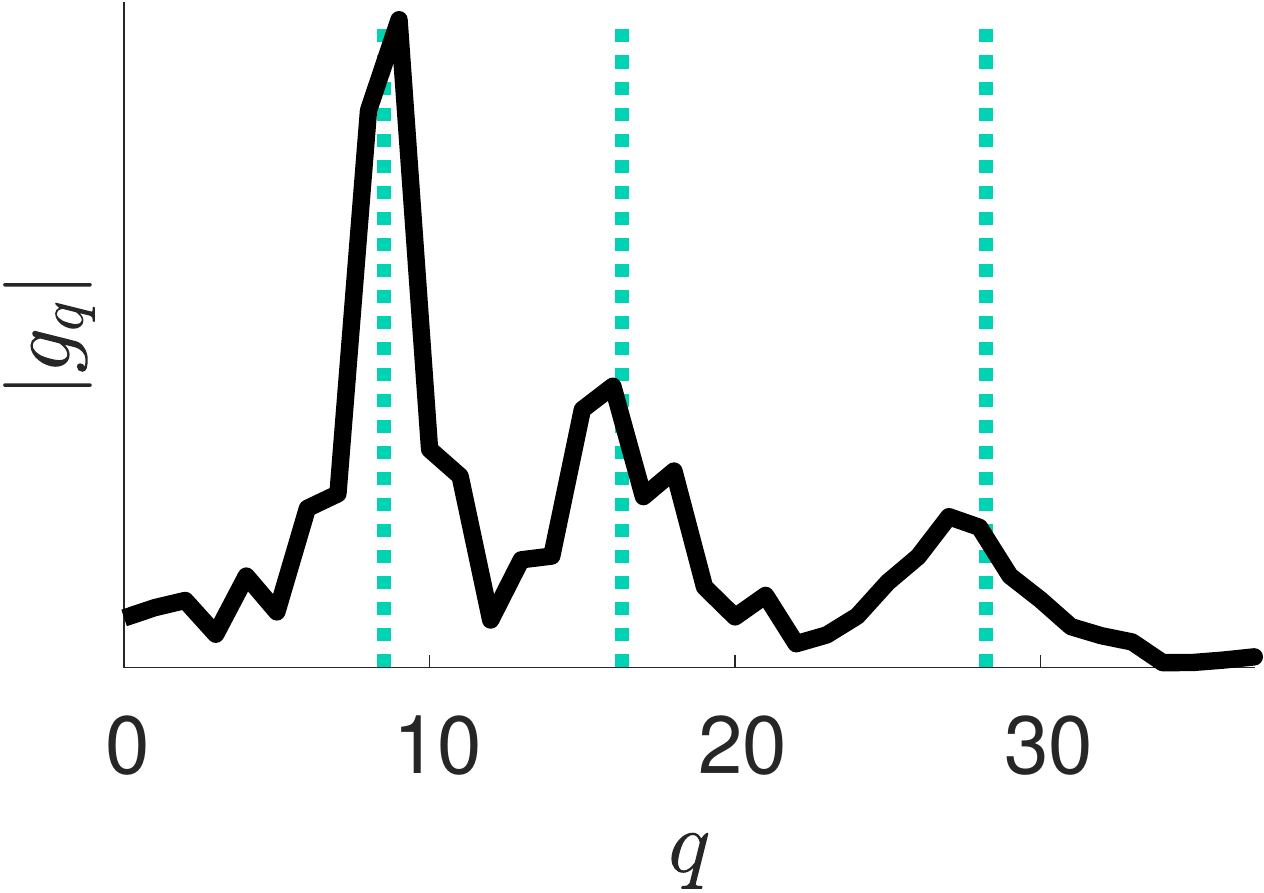}
    \includegraphics[height=2.6cm]{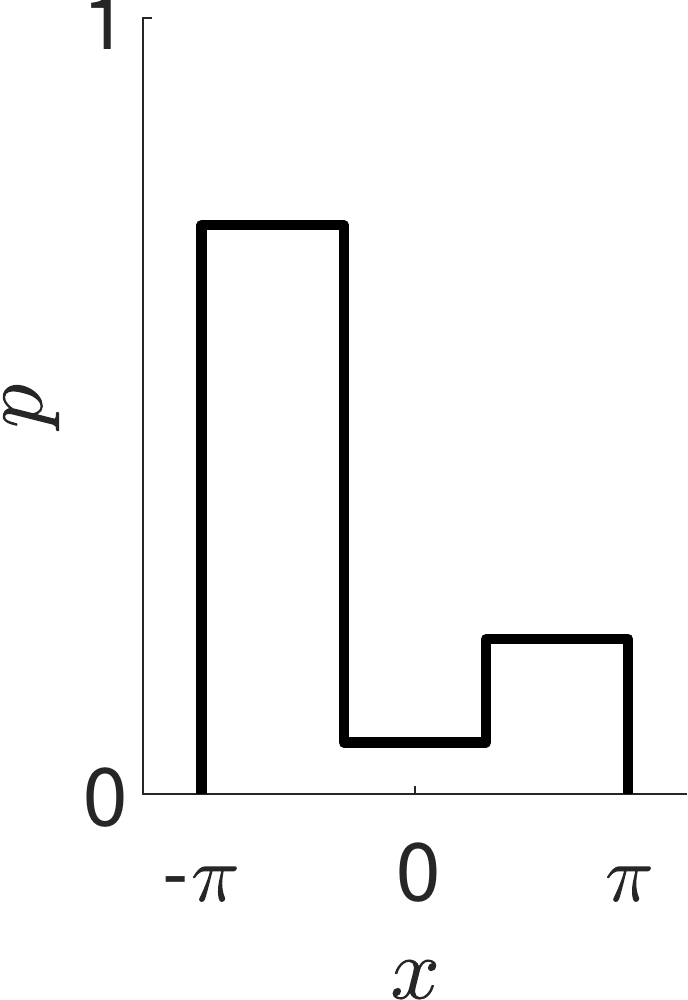}\\[0.2cm]
    \includegraphics[width=5.3cm]{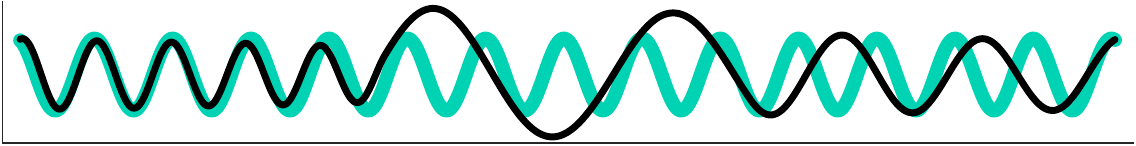}\\
    \includegraphics[width=5.3cm]{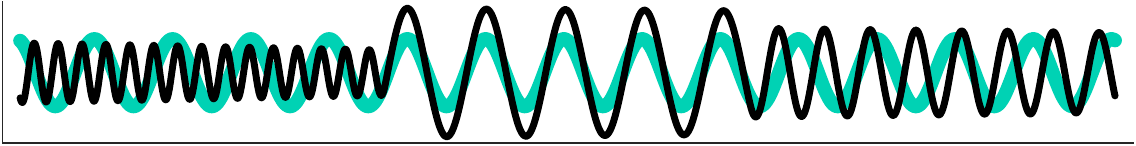}\\
    \includegraphics[width=5.3cm]{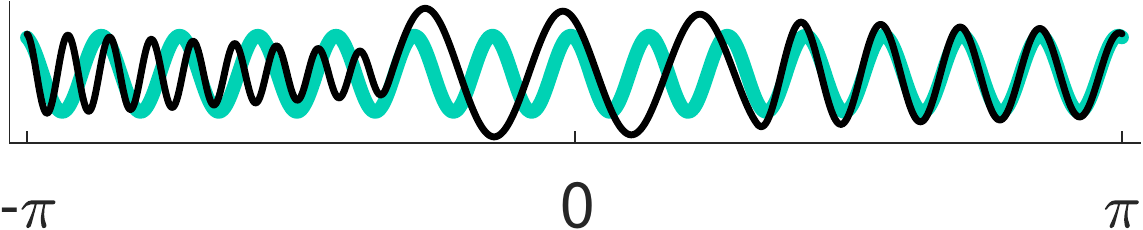}\\
    \caption{Illustration of Thm.~\ref{thm:convergence}. For a piecewise constant density with three regions (top right), a function $g(x) = \sin(14x)$ (in green, bottom plots) is projected onto the eigenfunctions of $k$ (three of which are shown with black curves in the bottom plots), producing coefficients $g_q$ (top left). This produces three peaks around the points predicted by our theory (marked by the dotted vertical lines), which correspond to high correlation of $g(x)$ with one of the three regions for the appropriate three basis functions (bottom row).}
\end{figure}

In Figure \ref{fig:times1D} we used the target function $g(x)=\sin(\kappa x)$ for different values of $\kappa$ to train a 2-layer network. The data was sampled from a non-uniform distribution with three constant regions of densities $3/(2\pi)(1/7,2/7,4/7)$. It can be seen that runtime increases for each region in proportion to $\kappa^2$, and the network converged faster at denser regions (in proportion to $p(x)$).

\begin{figure}[tb]    \label{fig:illustration}
    \centering
    \includegraphics[height=2.8cm]{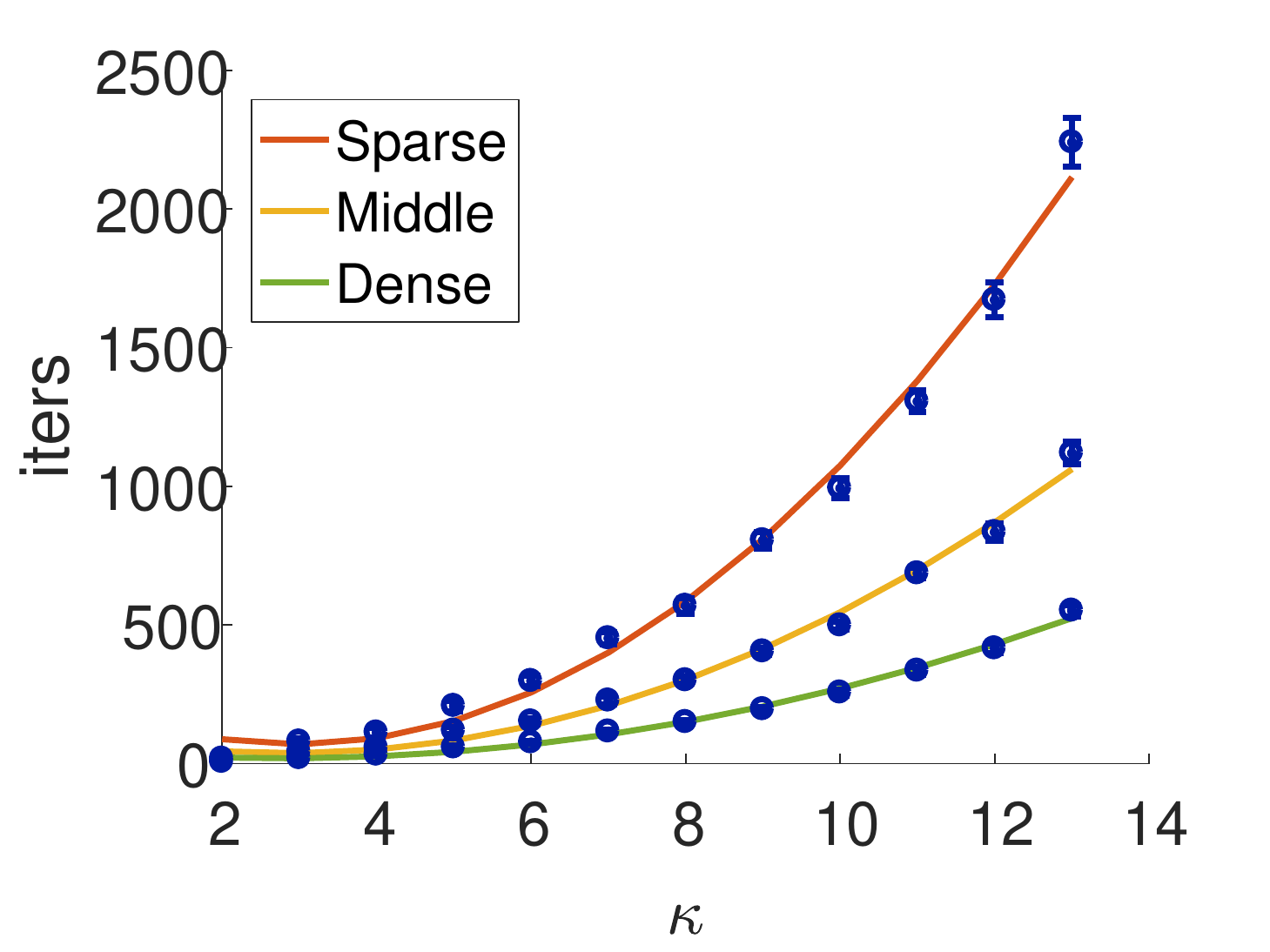} 
    \caption{Convergence times as a function of the target frequency $\kappa$ for a two-layer network trained with data drawn from a non-uniform distribution in $\Sphere^1$. We used the distribution of Figure \ref{fig:efun1d}, which is composed of three regions of constant density with a ratio of 1:2:4. For each region $R_j$ the network converges at time proportional to $\kappa^2/p_j$, as is indicated by the three quadratic curves fit to the data points. In addition, the median ratios between our measurements for the three regions are 1:1.96:3.89, in close fit to the distribution. 
    }
    \label{fig:times1D}
\end{figure}

\subsection{Higher dimension}

Deriving analytic expressions for data drawn from a non-uniform distribution in higher dimension, i.e., in $S^{d-1}$, $d>2$ is challenging and is left for future work. However, simulation experiments lead us to conjecture that the main properties in the $\Sphere^1$ hold also in higher dimension, i.e., (1) the eigenfunctions for piecewise constant distributions resemble concatenated patches of spherical harmonics, (2) the frequencies of these harmonics change with density, and increase monotonically as the respective eigenvalues become smaller, and (3) learning a harmonic function of frequency $k$ should require $O(k^d/p^*)$ iterations.

Figure \ref{fig:efun2d} shows an example plot of eigenfunctions in $\Sphere^2$ with a density function that is constant in each hemisphere. We further used harmonic functions of different frequencies to train a two-layer network with bias. Figure \ref{fig:times2D} shows convergence time as a function of frequency. As conjectured, for each region convergence time increases roughly in proportion to $k^3$, and convergence in different regions is linearly faster with density.

\section{Deep networks}  \label{sec:deep}

We next extend our discussion to NTK models of deep, fully connected networks. We first prove that the eigenvectors of NTK indeed characterize the convergence of GD of highly overparmeterized networks of \textit{finite} width. We then empirically investigate the eigenvectors and eigenvalues of NTK for data drawn from either uniform or non-uniform distributions and show convergence times for pure sine and harmonic target functions. 

We begin by showing that the eigenvectors of NTK characterize the dynamics of overparameterized FC networks of finite width. Our theorem extends Thm.~4.1 in \cite{arora2019fine} (see also \cite{cao2019}), which has dealt with two-layer networks, to deep nets. Consider a FC network of depth $L$ and width $m$ in each layer, and suppose the network is trained with $n$ pairs $\{(\x_i,y_i)\}_{i=1}^n$. Denote the vector of target values by $\y=(y_1,...,y_n)$ and the network predictions for these values at time $t$ by $\uu^{(t)}$. In our theorem, Thm.~\ref{thm:finegrained}, we use a slightly different model than the model stated above \eqref{model:arora}. First, we assume that the first and last layers are initialized and then held fixed throughout training, and the last layer is initialized randomly $\sim\mathcal{N}(0,\tau^2I)$. The NTK for this training data is summarized in an $n \times n$ matrix $H^\infty$, whose entries are set to $H^{\infty}_{ij}=k(\x_i,\x_j)$ where $k$ is defined in \eqref{eq:kernel_main}. Let $\vv_i$ and $\lambda_i$ respectively denote the eigenvectors of $H^\infty$ and their corresponding eigenvalues. The next Theorem establishes that the convergence rate of training this deep (finite width) network depends on the decomposition of the target values $\y$ over the eigenvectors of $H^\infty$.

\begin{theorem}  \label{thm:finegrained}
For any $\epsilon \in (0,1]$ and  $\delta \in (0,O(\frac{1}{L})]$, let $\tau=\Theta(\frac{\epsilon\hat{\delta}}{n})$, $m \geq \Omega \left( \frac{n^{24} L^{12} \log^5 m}{\delta^8 \tau^6} \right)$, $\eta = \Theta \left( \frac{\delta}{n^4 L^2 m \tau^2}\right)$. Then, with probability of at least $1-\hat{\delta}$ over the random initialization after $t$ GD iterations we have that \begin{equation}
\label{eq:finegrained}
    \|\y-\uu^{(t)}\| = \sqrt{\sum_{i=1}^n (1-\eta \lambda_i)^{2t}(\vv_i^T \y)^2} \, \pm \epsilon.
\end{equation}
\end{theorem}
The proof is provided in the supplementary material. Below, we give a brief proof sketch. 
First, we show that for any number of layers and at any iteration $t$ the following relation holds
\begin{align}\label{eq:epsilon}
\uu^{(t+1)}-\y=(I-\eta H(t))(\uu^{(t)}-\y)+\epsilon(t),
\end{align}
where 
$
    H_{ij}(t) =  \left< \frac{\partial f(\x_i,\w(t))}{\partial \w}, \frac{\partial f(\x_j,\w(t))}{\partial \w} \right>,
$
and the residual $\epsilon(t)$ due to the GD steps is relatively small. Then, based on several results due to \cite{allenzhu, arora}, we show that  $H(t)$  can be approximated by $H^{\infty}$, yielding, by applying recursion to \eqref{eq:epsilon} 
\begin{equation}\label{eq:main}
\uu^{(t)}-\y =
(I-\eta H^{\infty})^t(\uu^{(0)}-\y) + \xi(t).\\
\end{equation}
where $\|\xi(t)\| \leq O(\epsilon)$. Next we show that under the setting of $\tau$, $\| \uu^{(0)} \| \leq O(\epsilon)$. Finally, by applying the spectral decomposition to $H^{\infty}$ we obtain  
\eqref{eq:finegrained}.

\begin{figure}[tb]
    \centering
    \includegraphics[width=0.8\linewidth]{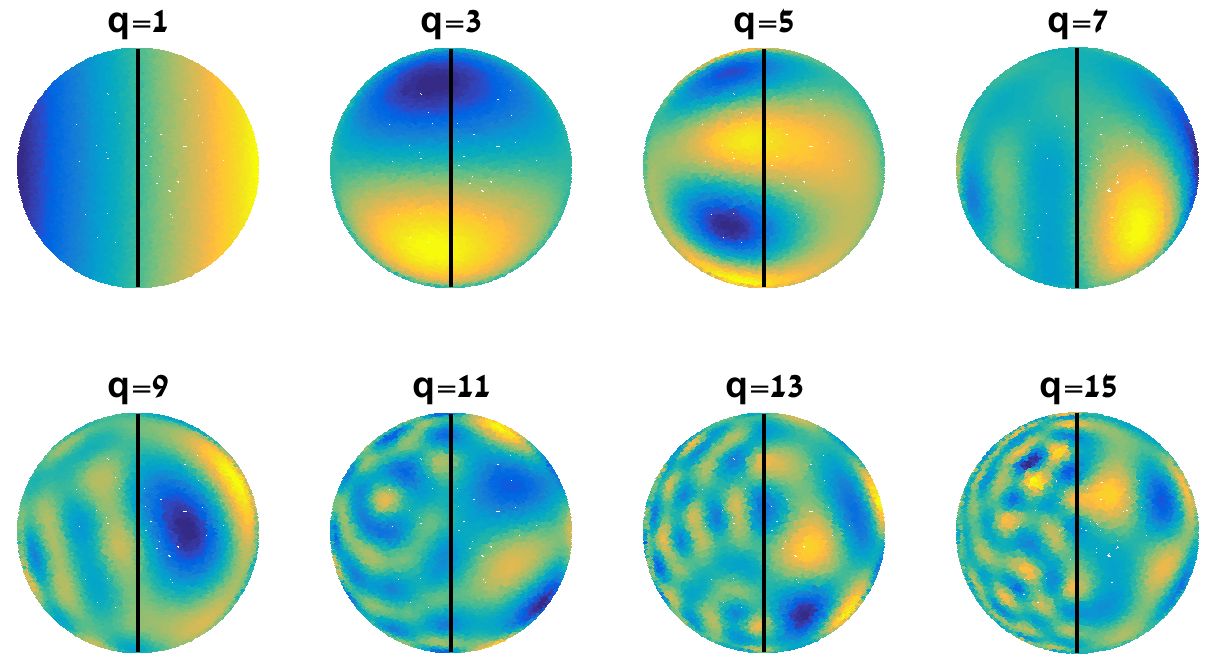} 
    \caption{The eigenfunctions of NTK for a two-layer network with bias for data drawn from a non-uniform distribution from $\Sphere^2$. The left and right hemispheres each have constant density with a ratio of 12:1.}
    \label{fig:efun2d}
\end{figure}

\begin{figure}[tb]
    \centering
    \includegraphics[height=2.0cm]{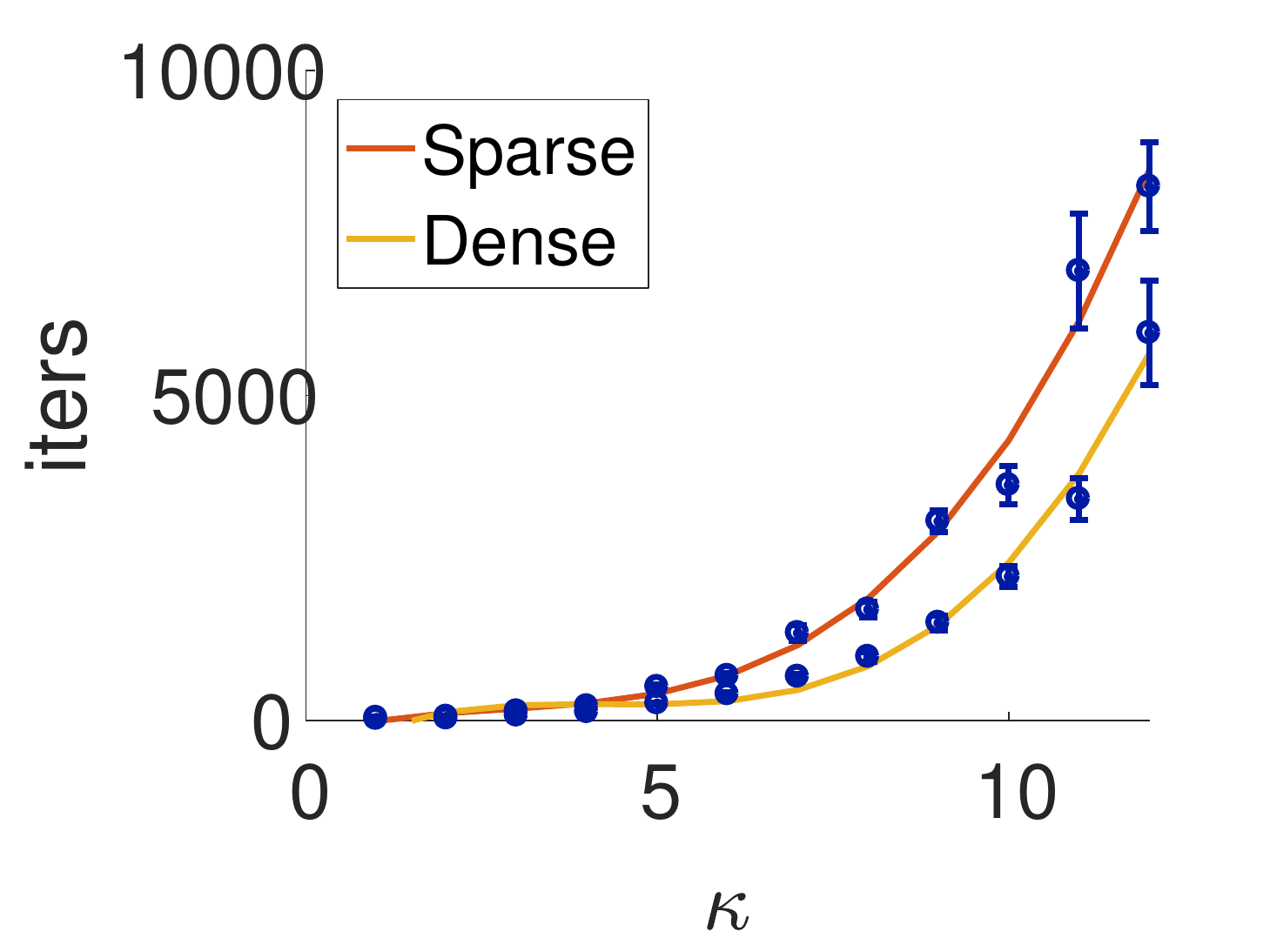} \includegraphics[height=2.0cm]{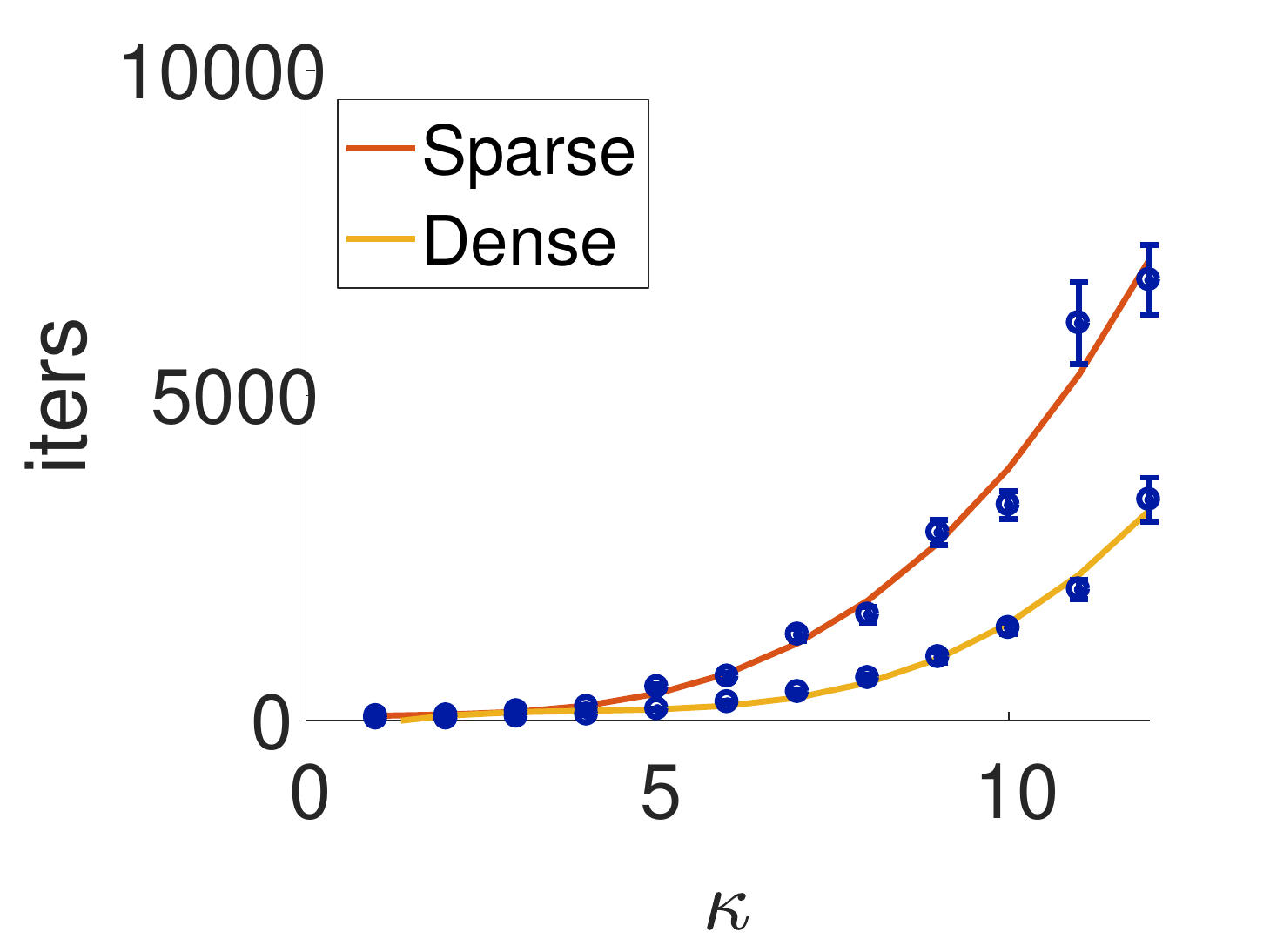}
    \includegraphics[height=2.0cm]{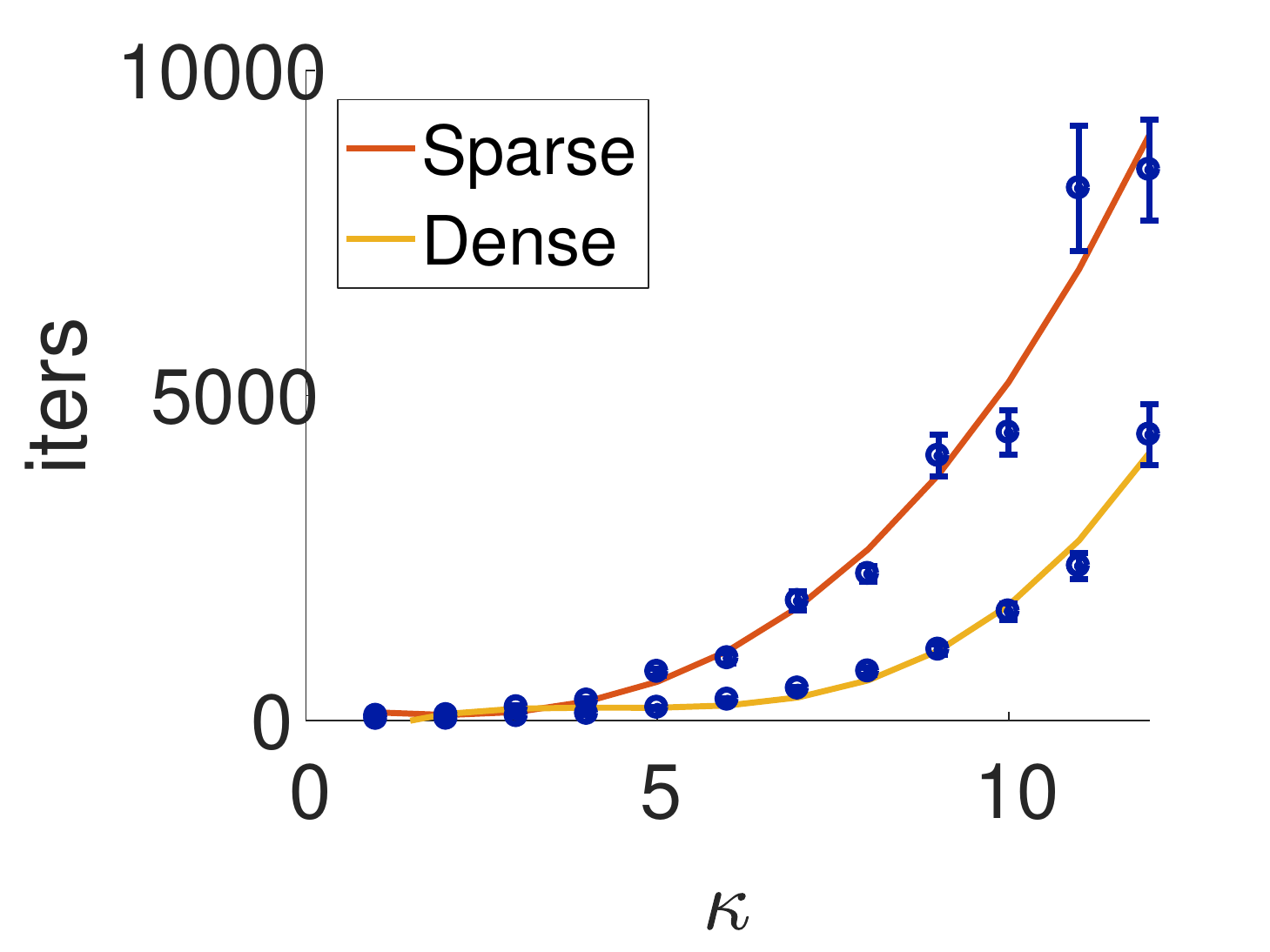}
    \caption{Convergence times as a function of the target harmonic frequency $\kappa$ for a two-layer network trained with data drawn from a non-uniform distribution in $\Sphere^2$. In each plot the sphere was divided into 2 halves, with density ratios (from left to right) of 1:2, 1:3, 1:4. The plot shows a cubic fit to the measurements. The median ratios between our measurements for the three subplots are 1.76, 2.45 and 2.99, undershooting our conjectured ratios. We believe this is due to sensitivity of experiments on $\Sphere^2$ to sampling.
    }
    \label{fig:times2D}
\end{figure}

\begin{figure}[tb]
    \centering
    \includegraphics[width=\linewidth]{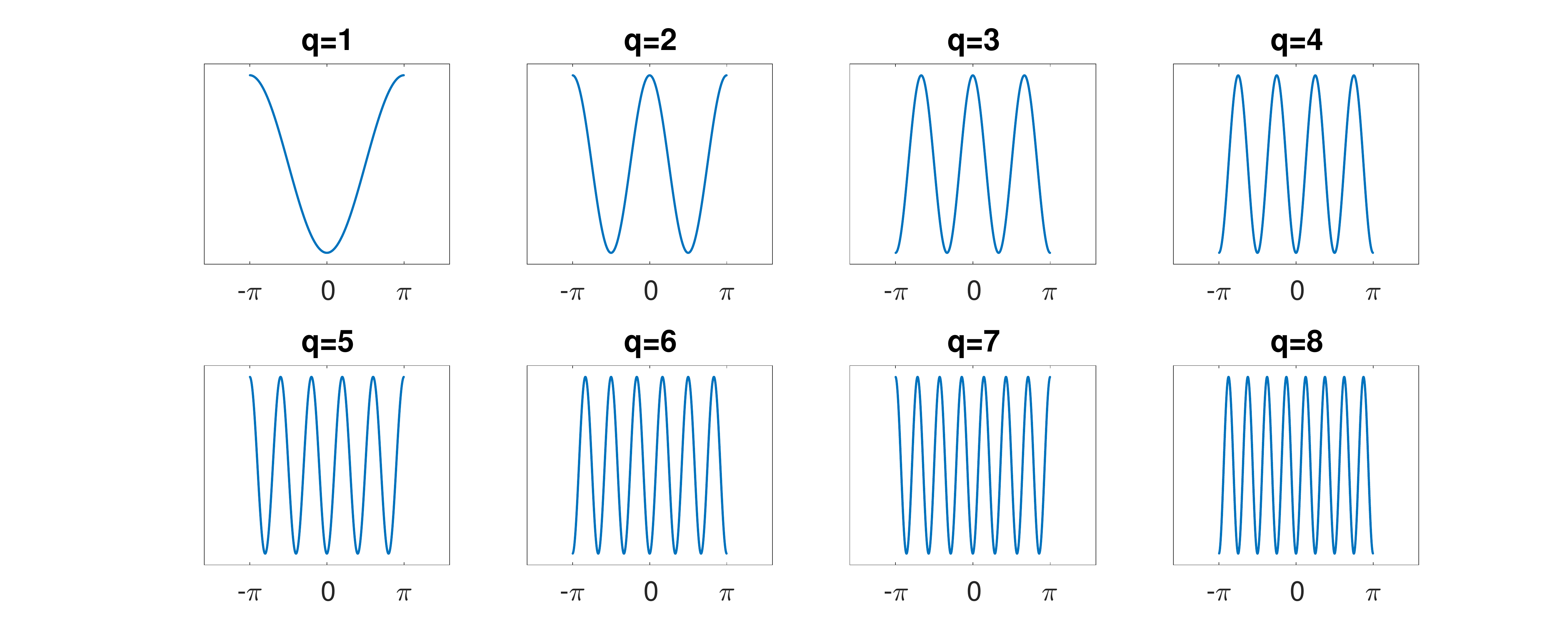} 
    \caption{The eigenfunctions of NTK for a deep network (depth 10) for the uniform distribution in $\Sphere^1$. The eigenvectors are arranged according to a descending order of their corresponding eigenvalues.}
    \label{fig:efundeep_uniform}
\end{figure}

Our next aim is to compute the eigenvectors and eigenvalues of NTK matrices for deep networks. This, together with Theorem \ref{thm:finegrained}, will allow us to derive convergence rates for different target functions. Toward that aim we observe that the NTK kernel $k(\x_i,\x_j)$ is a function of the inner product of its arguments. This can be concluded from its recursive definition in \eqref{eq:ntkdeep}, since $\Sigma^{(0)}(\x_i,\x_j)=\x_i^T\x_j$; both $\Sigma^{(h)}(\x_i,\x_j)$ and $\dot\Sigma^{(h)}(\x_i,\x_j)$ are (scaled) expectations over random variables drawn from a zero normal distribution and whose covariance, by recursion, is a function of the inner product $\x_i^T\x_j$. Consequently, the kernel decomposes over the zonal spherical harmonics in $\Sphere^{d-1}$ (or Fourier series in $\Sphere^1$), and for training data drawn from the uniform distribution the corresponding kernel matrix forms a convolution. 

Figure \ref{fig:efundeep_uniform} shows for the NTK of depth 10  that indeed the eigenvectors in $\Sphere^1$ is the Fourier series. We note that despite the lack of bias terms all the Fourier components are included. The eigenvalues decrease monotonically with frequency, indicating that the network should learn low frequency functions faster than high frequency ones. Moreover, as Figures \ref{fig:eigs_deep} and \ref{fig:exponent_uniform} show, regardless of depth, when trained with a function of frequency $\kappa$ overparameterized networks converge respectively at the asymptotic speed of $O(\kappa^2)$ and $O(\kappa^3)$ for uniform data in $\Sphere^1$ and $\Sphere^2$. Interestingly, however, the eigenvalues of NTK reveal a difference in the way deep and shallow networks treat low frequencies in the target function, as is refelcted by the plots in Figure \ref{fig:exponent_uniform}. Each line of one color represents the log of the eigenvalues for one network and the lines are ordered from shallow to deep in ascending order. The local slope of these lines indicate the speed of convergence for the corresponding frequencies. Asymptotically all the lines become parallel as the frequency $\kappa$ increases, implying that the asymptotic convergence times should be equal for all depths. However, for the low frequencies the lines corresponding to deeper networks are flatter than those corresponding to shallow networks. This flatter slope indicates that the frequency bias for such frequencies is smaller, implying that deep networks learn frequencies, e.g., 6-10, almost as fast as 1-5, while this is not true for shallow networks.

\begin{figure}[tb]
    \centering
    \includegraphics[height=2.6cm]{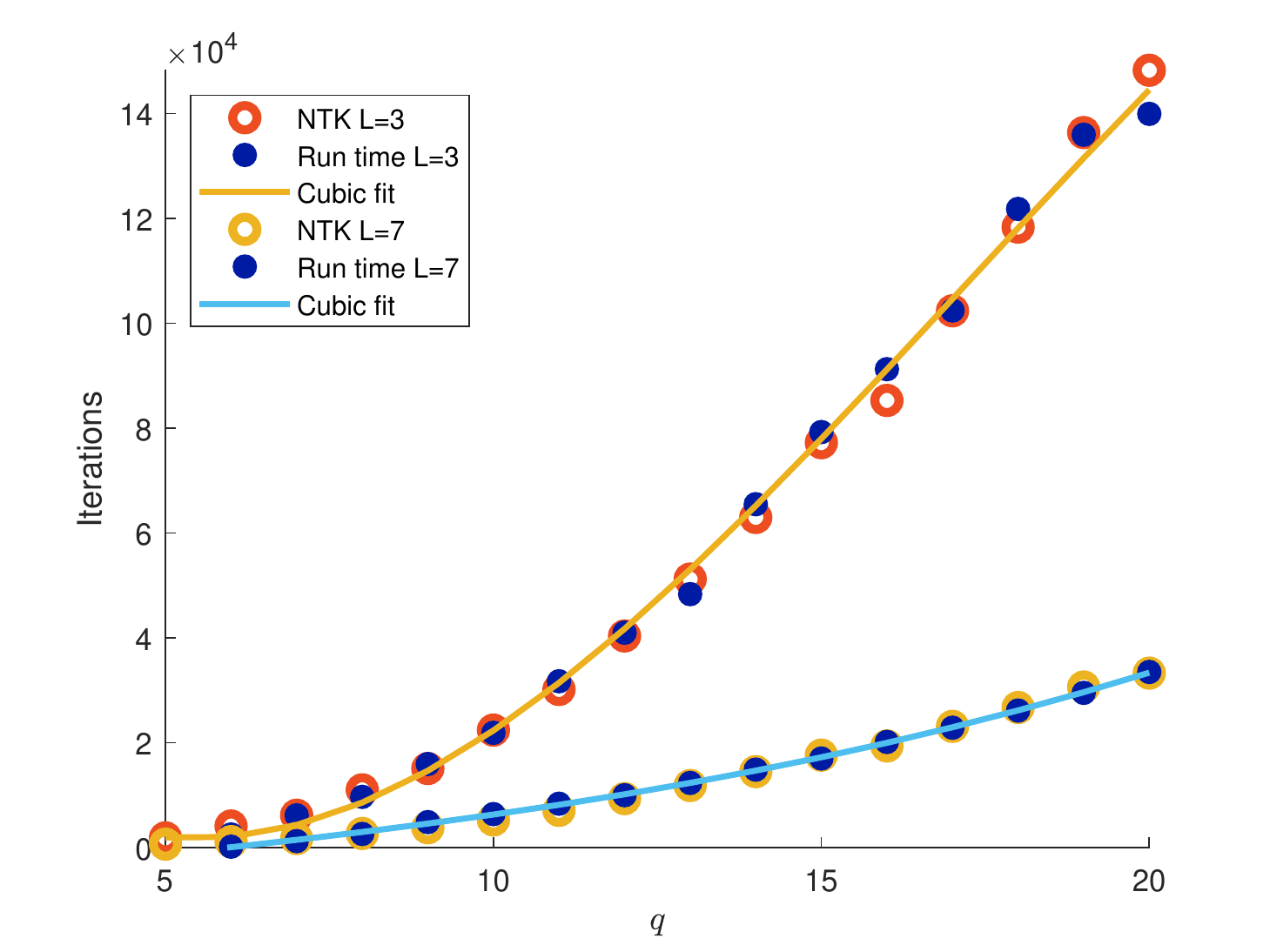}
    \includegraphics[height=2.6cm]{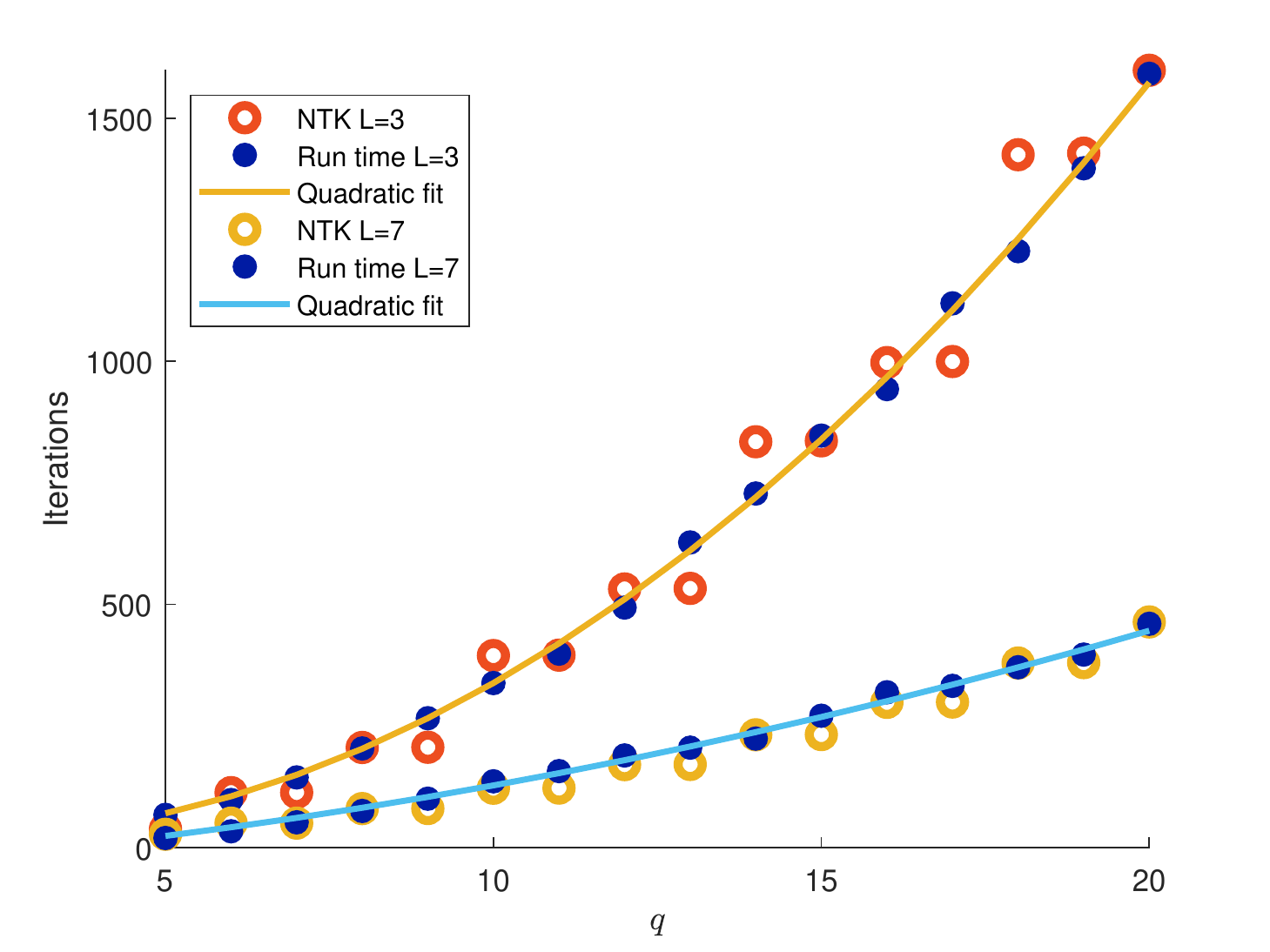}
    \caption{For deep networks (3 and 7 layers) and data drawn from the uniform distribution in $\Sphere^1$ (left) and $\Sphere^2$ (right) we plot training times as a function of target frequency (marked by the solid blue circles). This is compared to the times predicted by the eigenvalues of the corresponding NTK model (red circles).}
    \label{fig:eigs_deep}
%
    \centering
    \includegraphics[height=2.8cm]{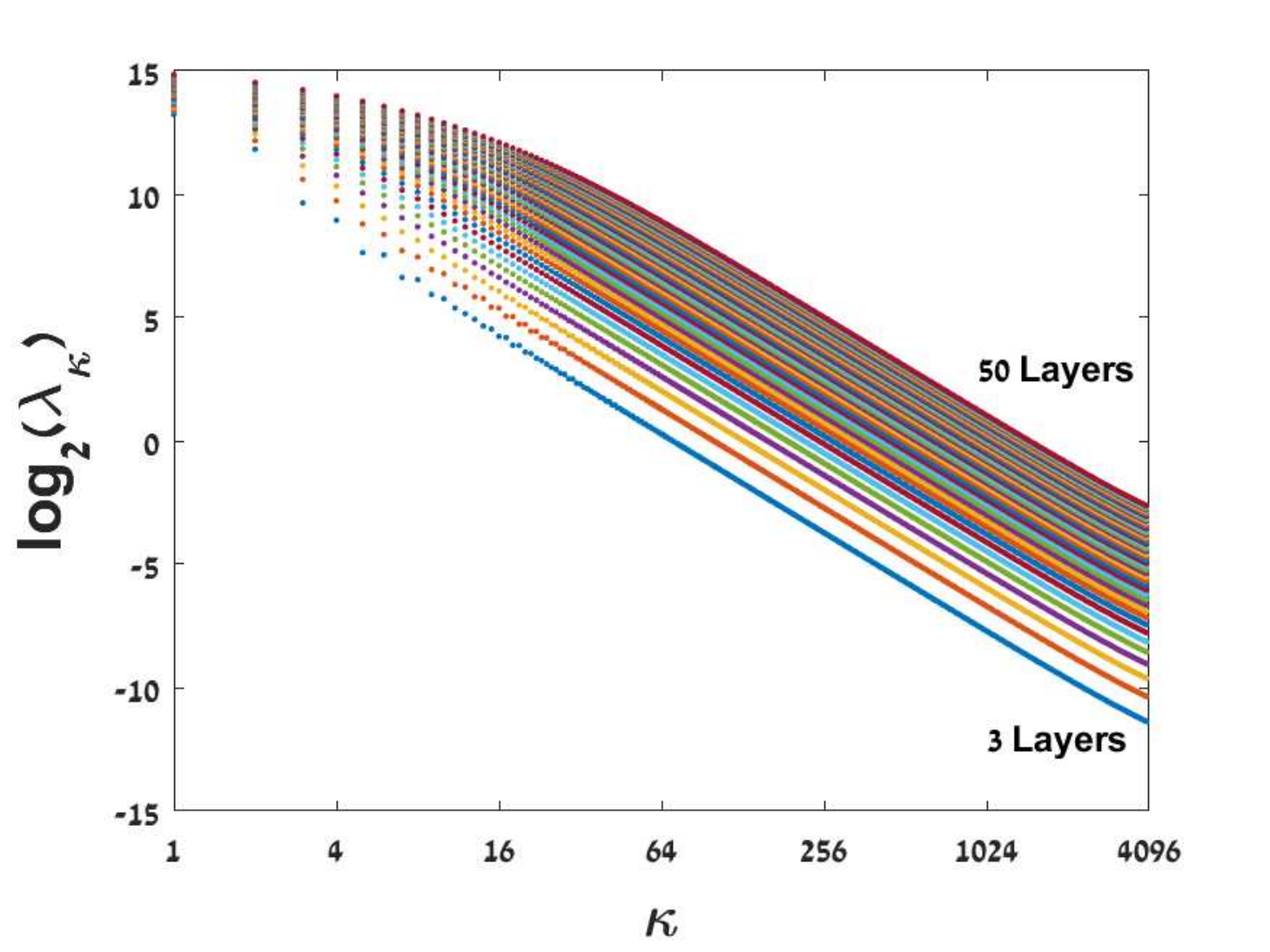}
    \includegraphics[height=2.8cm]{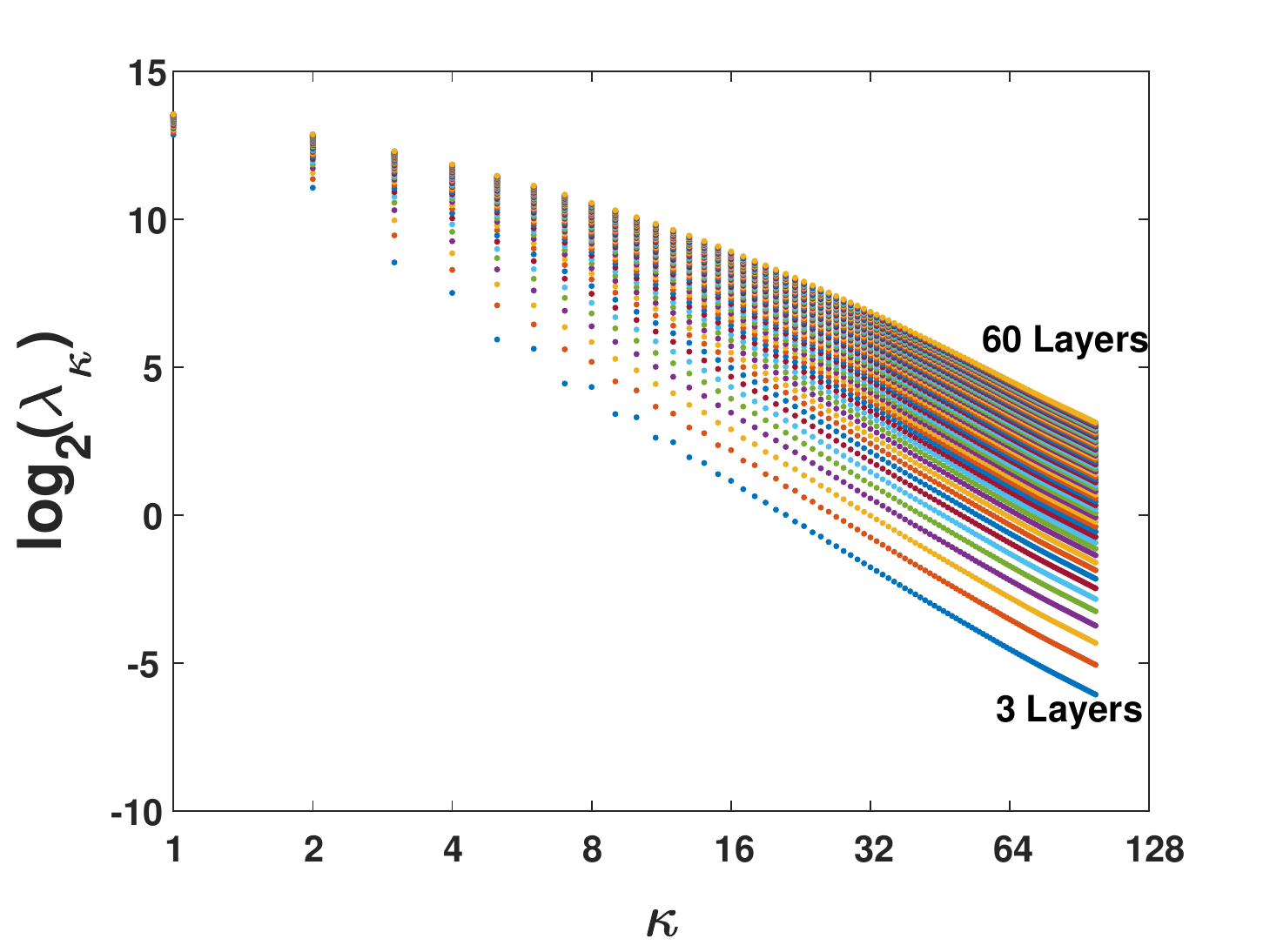}
    \caption{This figure shows a plot of the eigenvalues of NTK for FC networks of different depths with points drawn from a uniform density in $\Sphere^1$ (left) and $\Sphere^2$ (right). The plot is given in log-log scale. Networks of different depths are colored differently. Plots for deeper networks appear higher due to scaling. It can be seen that all curves decrease monotonically, indicating that the eigenvalues decay with frequency. In addition they all become parallel as the frequency $\kappa$ grows, converging to a slope of -2 for $\Sphere^1$ and -3 for $\Sphere^2$ (fitting the curves in the left plot starting at $\kappa=50$ yields a slope of 1.94; fitting the right plot starting at $\kappa=10$ yields a slope of 2.80). This indicates that asymptotically the rate of learning a frequency $\kappa$ is $O(\kappa^2)$ and $O(\kappa^3)$ respectively regardless of depth. The shallower slope of deep networks on the left part of each plot indicates that middle frequencies are learned faster with deep networks than with shallow ones.}
    \label{fig:exponent_uniform}
\end{figure}

Finally, for data drawn from a non-uniform distribution the eigenfunctions of NTK for deep networks appear to be indistinguishable from those obtained for two-layer networks. Figure \ref{fig:freq1d} shows a plot of the local frequencies obtained with NTK for a network of depth 10. It can be seen that the local frequencies are identical to those obtained with NTK for a two-layer network. The eigenvalues are similar to those obtained with the uniform density, up to a normalizing scale which depends on the distribution. Similarly to the two-layer case, learning a harmonic function of frequency $\kappa$ is therefore expected to require $O(\kappa^d/p^*)$ iterations.

\section{Conclusion}

The main contribution of our work is to show that insights about neural networks that have been derived with the assumption of uniformly distributed training data also apply, in interesting ways, to more realistic, non-uniform data.  Prior work has shown that the Neural Tangent Kernel provides a model of real, overparameterized neural networks that is tractable to analyze and that matches real experiments.  Our work shows that NTK has a frequency bias for non-uniform data distributions as well as for uniform ones.  This strengthens the case that this frequency bias may play an important role in real neural networks.

We also quantify this frequency bias. We derive an expression for the eigenfunctions of NTK, showing that for piecewise constant data distributions the eigenfunctions consist of piecewise harmonic functions.  The frequency of these piecewise functions increases linearly with the square root of the local density of the data. As a consequence, for 1D inputs, networks modeled by NTK learn harmonic functions with a speed that increases quadratically in their frequency and decreases linearly with the local density.  Experiments indicate that these results generalize naturally to higher dimensions.  These results support the idea that overparameterized networks avoid overfitting because they fit target functions with smooth functions, and are slow to add high frequency components that could overfit.

\section*{Acknowledgements}

This material is based partly upon work supported by the National Science Foundation under Grant No. DMS1439786 while the authors were in residence at the Institute for Computational and Experimental
Research in Mathematics in Providence, RI, during the Computer Vision program. We would like to thank the Quantifying Ensemble Diversity for Robust Machine Learning (QED for RML) program from DARPA for their support of this project.

\bibliography{nonuniform}
\bibliographystyle{icml2020}

\onecolumn


\appendix

\section{Eigenfunctions of NTK for a two layer-network for data drawn from a piecewise constant distribution}

\begin{lemma}
Let $p(x)$ be a piecewise constant density function on $\Sphere^1$. Then the eigenfunctions in Eq.~(9) in the paper 
satisfy the following ordinary differential equation
\begin{equation}  \label{eq:app_solution}     f''(x)=-\frac{p(x)}{\pi\lambda}f(x).
\end{equation}
\end{lemma}

\begin{proof}
Combining Eqs.~(9) and (10) in the paper 
we have
\begin{equation}  \label{eq:eigfofk}
\int_{x-\pi}^{x+\pi} (1+\cos(z-x))(\pi-|z-x|) f(z)p(z)dz = 4\pi \lambda f(x)
\end{equation}

Below we take six derivatives of \eqref{eq:eigfofk} with respect to $x$. We use parenthesized superscripts $f^{(n)}(x)$ to denote the $n^\mathrm{th}$ derivative of $f$ at $x$. First derivative
\begin{eqnarray*}
    4\pi \lambda f^{(1)}(x) &=& 
    -\int_{x-\pi}^{x} \left(1 + \cos(z-x)-(\pi + z - x)\sin(z-x) \right) f(z)p(z)dz \nonumber\\
    &&+\int_{x}^{x+\pi} \left(1 + \cos(z-x) +(\pi -z +x) \sin(z-x) \right) f(z)p(z)dz
\end{eqnarray*}
Second derivative
\begin{eqnarray*}
    4 \pi \lambda f^{(2)}(x) + 4f(x)p(x) &=& 
    -\int_{x-\pi }^{x} \left(2\sin(z-x) +(\pi +z -x) \cos(z-x)\right) f(z)p(z)dz \nonumber\\
    &&+\int_{x}^{x+\pi} \left(2\sin(z-x) -(\pi -z +x)\cos(z-x)\right) f(z)p(z)dz
\end{eqnarray*}
Adding this to \eqref{eq:eigfofk}
\begin{eqnarray}  \label{eq:2ndd}
    4\pi \lambda f^{(2)}(x) + 4 f(x) p(x) + 4 \pi \lambda f(x) &=&
    \int_{x-\pi}^{x} \left(\pi+z-x -2\sin(z-x)\right) f(z)p(z)dz \nonumber\\
    &&+\int_{x}^{x+\pi} \left(\pi-z+x +2\sin(z-x)\right) f(z)p(z)dz
\end{eqnarray}
Third derivative
\begin{eqnarray*}
    && 4\pi \lambda f^{(3)}(x) + 4\pi \lambda f^{(1)}(x) + 4 f^{(1)}(x) p(x) + 4 f(x) p^{(1)}(x) = \nonumber\\
    && \int_{x-\pi}^{x} \left(2\cos(z-x)-1\right) f(z)p(z)dz
    -\int_{x}^{x+\pi} \left(2\,\cos(z-x)-1\right) f(z)p(z)dz
\end{eqnarray*}
Fourth derivative
\begin{eqnarray*}
    && 4\pi \lambda f^{(4)}(x) + 4\pi \lambda f^{(2)}(x) +4 f^{(2)}(x) p(x) + 8f^{(1)}(x) p^{(1)}(x) + 4f(x) p^{(2)}(x) - 2f(x)p(x) = \nonumber\\ && 3f(x-\pi) p(x-\pi) + 3f(x+\pi) p(x+\pi) -\int_{x}^{x+\pi} 2\sin(z-x) f(z)p(z) dz + \int_{x-\pi}^{x} 2\sin(z-x) f(z)p(z) dz
\end{eqnarray*}
Adding this to \eqref{eq:2ndd}
\begin{eqnarray*}
    && 4\pi \lambda f^{(4)}(x) + 8\pi \lambda f^{(2)}(x) + 4\pi \lambda f(x) + 2f(x)p(x) + 4p(x) f^{(2)}(x)
    +8 f^{(1)}(x) p^{(1)}(x) +4f(x) p^{(2)}(x) = \nonumber\\
    && 3f(x-\pi) p(x-\pi) +3 f(x+\pi) p(x+\pi) +\int_{x}^{x+\pi} (\pi-z+x) f(z)p(z) dz +\int_{x-\pi}^{x} (\pi+z-x) f(z)p(z) dz
\end{eqnarray*}
Fifth derivative
\begin{eqnarray*}
    && 4\pi \lambda f^{(5)}(x) + 8\pi \lambda f^{(3)}\left(x\right) + 4\pi \lambda f^{(1)}(x) + 4f^{(3)}(x) p(x) + f^{(2)}(x) p^{(1)}(x) + 12 f^{(1)}(x) + p^{(2)}(x) \nonumber\\
    && + 2f^{(1)}(x) p(x) + 4f(x) p^{(3)}(x) = -2f(x) p^{(2)}(x) +3f^{(1)}(x-\pi) p(x-\pi) +3f(x-\pi)p^{(1)}(x-\pi) \nonumber\\
    && +3f^{(1)}(x+\pi)p(x+\pi) +3f(x+\pi) p^{(1)}(x+\pi) -\int_{x-\pi}^{x} f(z)p(z) dz + \int_{x}^{x+\pi} f(z)p(z) dz
\end{eqnarray*}
Sixth derivative
\begin{eqnarray*}
    && 4\pi \lambda f^{(6)}(x) + 8\pi \lambda f^{(4)}(x) + 4\pi \lambda f^{(2)}(x) = 3f^{(2)}(x+\pi) p\left(x+\pi \right) +3p^{(2)}(x+\pi) f(x+\pi) \nonumber\\
    && +6f^{(1)}(x+\pi) p^{(1)}(x+\pi) -2f(x) p(x) +f(x-\pi) p(x-\pi) -4f(x) p^{(4)}(x) -4p(x) f^{(4)}\left(x\right) \nonumber\\
    && -2f(x) p^{(2)}(x) -2\,p(x)\, f^{(2)}(x) +f(x+\pi) p(x+\pi) +6f^{(1)}(x-\pi) p^{(1)}(x-\pi) +3f^{(2)}(x-\pi) p(x-\pi)\nonumber\\ && +3p^{(2)}(x-\pi) f(x-\pi) -16 f^{(1)}(x) p^{(3)}(x) -16 f^{(3)}(x) p^{(1)}(x) -24 p^{(2)}(x) f^{(2)}(x) -4 f^{(1)}(x) p^{(1)}(x)
\end{eqnarray*}
Next, we simplify and rearrange. We omit dependence on $x$, note that $f(x-\pi)=f(x+\pi)$ and $p(x-\pi)=p(x+\pi)$ and respectively denote them by $\bar f$ and $\bar p$. 
\begin{align*}
    & 2\pi \lambda f^{(6)} +2(p +2\pi \lambda) f^{(4)} +8 p^{(1)} f^{(3)} +(p +12 p^{(2)} +2\pi \lambda) f^{(2)} + \nonumber\\ 
    & 2(p^{(1)} +4p^{(3)})f^{(1)} 
    +(p +p^{(2)} +2p^{(4)})f =
    (\bar p +3\bar p^{(2)})\bar f +6\bar p^{(1)} \bar f^{(1)} +3\bar p \bar f^{(2)}
\end{align*}

Assume next that $p(x)$ is constant around $x$ and $x-\pi$, so its derivatives at these points vanish. Then,
\begin{align*}
    2\pi \lambda f^{(6)} +(2p +4\pi \lambda) f^{(4)} +(p +2\pi \lambda) f^{(2)} +pf =
    \bar p \bar f +3\bar p\bar f^{(2)}
\end{align*}
We next make the assumption that $p(x)$ has a period of $\pi$ (so $p=\bar p$) in which case $f(x+\pi)=-f(x)$ (i.e., $\bar f=-f$). These assumptions will be removed later. With these assumptions we have
\begin{align*}
    2\pi \lambda f^{(6)} +(2p +4\pi \lambda) f^{(4)} +(4p +2\pi \lambda) f^{(2)} + 2pf = 0
\end{align*}
It can be readily verified that this equation is solved by \eqref{eq:app_solution}.

Finally, if $p(x)$ does not have a period of $\pi$ we can preprocess the data in a straightforward way to make $p$ have a period of $\pi$ (by mapping the interval $[0,4\pi)$ to $[0,2\pi)$) without changing the function that needs to be learned.
\end{proof}

\section{The amplitudes of the eigenfunctions in different regions}

In this section for the NTK of a 2-layer network for which only the first layer is trained we compute bounds on the amplitudes of its eigenfunctions. We first bound the ratios between the amplitudes in two neighboring regions, and use this in the following section to bound the amplitude in any one region.

\subsection{Ratios between the amplitudes of neighboring regions}
If $p(x)=p_j$ is constant in each region $R_j \subseteq \Sphere^1$, $1\le j\le l$, then the eigenfunction or order $q$ $f_q(x)$ for $x \in R_j$ can be written as
\begin{equation*}
    f_q(x)=a_j \cos\left(\frac{q\sqrt{p_j}x}{Z}+b_j\right)
\end{equation*}
where $a_j \ge 0$. In this part we characterize the amplitudes the different regions $a_j$ for $j= 1, ..., l$. 

We notice that the eigenfunctions appear to be continuous and differentiable. Without loss of generality, assume that the boundary between region $j$ to region $j+1$ happens at $x=0$. Then the eigenfunction in the vicinity of 0 is defined as follows: 
\begin{equation*}
    f_{q}(x)=\begin{cases} a_j\cos(q\frac{\sqrt{p_j}}{Z}x+b_j) & x\leq 0  \\ a_{j+1}\cos(q\frac{\sqrt{p_{j+1}}}{Z}x+b_{j+1}) & x\geq 0  \end{cases}
\end{equation*}
Continuity at $x=0$ implies that 
\begin{equation} \label{eq::cont_constaint}
    a_j\cos(b_j)=a_{j+1}\cos(b_{j+1}) \, \Rightarrow \, \frac{a_j}{a_{j+1}}=\frac{\cos(b_{j+1})}{\cos(b_j)}
\end{equation}
Differentiability at $x=0$ implies
\begin{equation*}
    a_j\sqrt{p_j}\sin(b_j)=a_{j+1}\sqrt{p_{j+1}}\sin(b_{j+1}) \, \Leftrightarrow \, \frac{a_j}{a_{j+1}} = \frac{\sqrt{p_{j+1}} \sin(b_{j+1})}{\sqrt{p_j}\sin(b_j)}
\end{equation*}

These allow us to bound the ratio  $a_{j}/a_{j+1}$. We have
\begin{equation}\label{eq::derrive_amp}
    \frac{a_j}{a_{j+1}}=\frac{\sqrt{p_{j+1}}\sin(b_{j+1})}{\sqrt{p_j}\sin(b_j)} \, \Rightarrow \, \left(\frac{a_j}{a_{j+1}}\right)^2 = \frac{p_{j+1}\sin^2(b_{j+1})}{p_j\sin^2(b_j)} = \frac{p_{j+1}(1-\cos^2(b_{j+1}))}{p_j(1-\cos^2(b_j))}
\end{equation}
On the other hand, from \eqref{eq::cont_constaint} we know that 
\begin{equation}\label{eq::continuity_amp}
    \frac{a_j}{a_{j+1}}=\frac{\cos(b_{j+1})}{\cos(b_{j+1})} \, \Rightarrow \, \left(\frac{a_j}{a_{j+1}}\right)^2 = \frac{\cos^2(b_{j+1})}{\cos^2(b_j)} \, \Rightarrow \, \cos^2(b_{j+1})=\cos^2(b_j)\left(\frac{a_j}{a_{j+1}}\right)^2
\end{equation}
Substitute \eqref{eq::continuity_amp} in \eqref{eq::derrive_amp} we get
\begin{equation*}
    \left(\frac{a_j}{a_{j+1}}\right)^2 = \frac{p_{j+1}}{p_j}\frac{1-\cos^2(b_j)(\frac{a_j}{a_{j+1}})^2}{1-\cos^2(b_j)} \, \Rightarrow \, \left(\frac{a_j}{a_{j+1}}\right)^2(1-\cos^2(b_j)) = \frac{p_{j+1}}{p_j}\left(1-\cos^2(b_j)\left(\frac{a_j}{a_{j+1}}\right)^2\right)
\end{equation*}
And we have
\begin{equation*}
\left(\frac{a_j}{a_{j+1}}\right)^2(1-\cos^2(b_j)+\frac{p_{j+1}}{p_j}\cos^2(b_j))=\frac{p_{j+1}}{p_j}
\end{equation*}
implying that
\begin{equation}
\label{eq::amp_ratio_squared}
    \left(\frac{a_j}{a_{j+1}}\right)^2=\frac{\frac{p_{j+1}}{p_j}}{1-\cos^2(b_j)\left(1-\frac{p_{j+1}}{p_j}\right)}
\end{equation}
WLOG assume that $p_{j+1}/p_j \ge 1$ then
\begin{equation*}
    \cos^2(b_j)\left(1- \frac{p_{j+1}}{p_j}\right) \leq 0 \,\Rightarrow \,
    \frac{1}{1-\cos^2(b_j)\left(1- \frac{p_{j+1}}{p_j}\right)} \leq 1
\end{equation*}
As a result we get
\begin{equation*}
    \left(\frac{a_j}{a_{j+1}}\right)^2 = \frac{\frac{p_{j+1}}{p_j}}{1-\cos^2(b_j)(1-\frac{p_{j+1}}{p_j})}\leq \frac{p_{j+1}}{p_j} \, \Rightarrow \, \frac{a_j}{a_{j+1}}\leq\sqrt{\frac{p_{j+1}}{p_j}}
\end{equation*}

For a lower bound note that the denominator in \eqref{eq::amp_ratio_squared} satisfies
\begin{equation*}
    1-\cos^2(b_j)(1-\frac{p_{j+1}}{p_j}) = \sin^2(b_j) + \frac{p_{j+1}}{p_j} \cos^2(b_j) \le \frac{p_{j+1}}{p_j}
\end{equation*}
where the inequality is due to the assumption that $p_{j+1} \ge p_j$. Consequently, $(a_{j+1}/a_j)^2 \ge 1$. In summary, we have bounded the ratios between the amplitudes of neighboring regions by
\begin{equation}
    1 \leq \frac{a_j}{a_{j+1}}\leq \sqrt{\frac{p_{j+1}}{p_j}}
\end{equation}

We next note that these bounds are tight and are obtained in the following setup. Assume we have an even number of regions of constant density $l$ each with equal size. Suppose that in each region the eigenfunction includes an integer number of cycles. For each $q$ we construct an eigenfunction, by choosing a phase $b_j=0$ for $j=1, ..., l$, and it holds that the border between region $l/2$ and $l/2+1$ lies at $x=0$. As a result, at this point we have 
\begin{equation*}
    a_{\frac{l}{2}}\cos\left(\frac{q\sqrt{p_{\frac{l}{2}}}\, 0}{Z}\right) = a_{\frac{l}{2}+1}\cos\left(\frac{q\sqrt{p_{\frac{l}{2}+1}}\, 0}{Z}\right) \, \Rightarrow \, a_{\frac{l}{2}} = a_{\frac{l}{2}+1}
\end{equation*} 
But since each region contains an integer number of cycles we get for $j=1, ..., l$
\begin{equation}\label{eq::borders_amplitudes}
    \cos\left(\frac{q\sqrt{p_{\frac{l}{2}}}\, 0}{Z}\right) = \cos\left(\frac{q\sqrt{p_{j}}}{Z}\left(\frac{2\pi}{l}j-\pi\right)\right)=1
\end{equation}
Continuity implies for $j=2, ..., l$
\begin{equation*}
    a_{j-1}\cos\left(\frac{q\sqrt{p_{j-1}}}{Z}\left(\frac{2\pi(j-1)}{l}-\pi\right)\right) = a_{j}\cos\left(\frac{q\sqrt{p_{j}}}{Z}\left(\frac{2\pi(j-1)}{l}-\pi\right)\right)\Rightarrow a_{j-1}=a_{j}
\end{equation*}
As a result, for each $q$ we get one eigenfunction (up to a global scale)
\begin{equation}\label{eq:tight_cos}
    f_{q}^1(x) = \cos\left(\frac{q\sqrt{p_j}x}{Z}\right) \ \ \text{, for}   \ \  x\in\left[\frac{2\pi(j-1)}{l}-\pi,\frac{2\pi j}{l}-\pi\right]
\end{equation}
We next construct a second eigenfunction for each $q$. Since there is an integer number of cycles in each region, to keep the second eigenfunction of each $q$ orthogonal to the first one, we choose a phase of $-\pi/2$:
\begin{equation*}
    f_{q}^1(x)=     a_j\sin\left(\frac{q\sqrt{p_j}x}{Z}\right) \ \ \text{, for}   \ \  x\in\left[\frac{2\pi(j-1)}{l}-\pi,\frac{2\pi j}{l}-\pi\right]
\end{equation*}
Next, to maintain differentiability, the derivative at the border between regions $R_j$ and $R_{j+1}$ must be equal. So at $x=2\pi j/l-\pi$ we have for $j=1, ..., l-1$
\begin{equation*}
   \frac{d}{dx}\left(a_{j}\sin\left(\frac{q\sqrt{p_{j}}x}{Z}\right)\right) =  \frac{d}{dx}\left(a_{j+1}\sin\left(\frac{q\sqrt{p_{j+1}}x}{Z}\right)\right) \Rightarrow
\end{equation*}
\begin{equation*}
    -\frac{a_{j}q\sqrt{p_{j}}}{Z} \cos\left(\frac{q\sqrt{p_{j}}x}{Z}\right) = -\frac{a_{j+1}q\sqrt{p_{j+1}}}{Z} \cos\left(\frac{q\sqrt{p_{j+1}}x}{Z}\right)\Rightarrow
\end{equation*}
\begin{equation*} 
    a_{j}\sqrt{p_{j}} \cos\left(\frac{q\sqrt{p_{j}}x}{Z}\right) = a_{j+1}\sqrt{p_{j+1}} \cos\left(\frac{q\sqrt{p_{j+1}}x}{Z}\right) 
\end{equation*}
From \eqref{eq::borders_amplitudes}  we have
\begin{equation*} 
    a_{j}\sqrt{p_{j}}= a_{j+1}\sqrt{p_{j+1}} \Rightarrow  \frac{a_{j}}{a_{j+1}}= \frac{\sqrt{p_{j+1}}}{\sqrt{p_{j}}} 
\end{equation*}
And we can choose for the second eigenfunction for each $q$ (up to a global scale)
\begin{equation} \label{eq:tight_sin}
    f_{q}^2(x) = \frac{1}{\sqrt{p_{j}}}\sin\left(\frac{q\sqrt{p_j}x}{Z}\right) \ \ \text{, for}   \ \  x\in\left[\frac{2\pi(j-1)}{l}-\pi,\frac{2\pi j}{l}-\pi\right]
\end{equation}
In Figure \ref{fig:efun1dtight} we show an example for this setup.

\begin{figure}[tb]
    \centering
    \includegraphics[width=\linewidth]{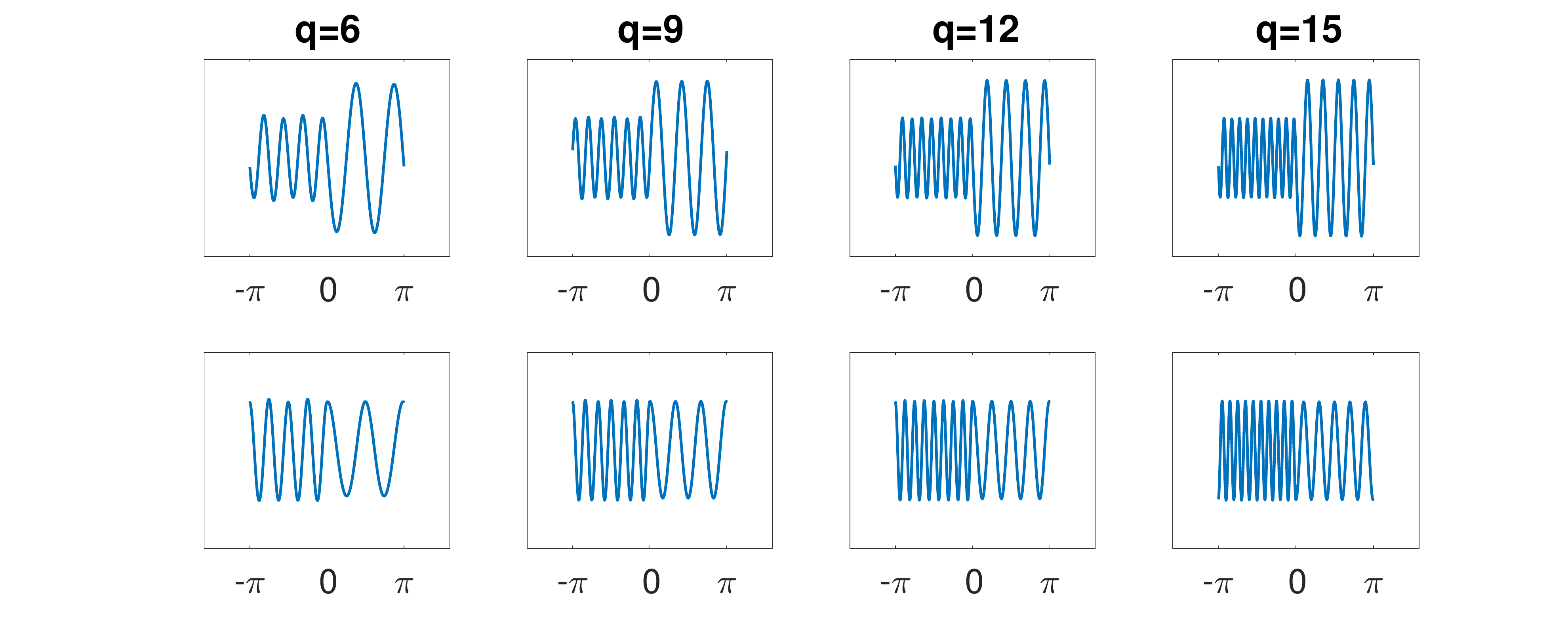}
    \caption{For the NTK of a two-layer network with bias we plot in each of the four columns four of its eigenfunction pairs (each of the same eigenvalue) under a non-uniform data distribution of $p(x) \in 1/\pi\{4/5,1/5\}$ in $\Sphere^1$. For this distribution whenever $\mathrm{mod}(q,3)=0$ there is an integer number of cycles in each region. As a result, for each $q$ we obtain two eigenfunctions of the form of \eqref{eq:tight_cos} and \eqref{eq:tight_sin}.}
    \label{fig:efun1dtight}
    \end{figure}

\subsection{Bounding $a_j$ }

Assuming $p(x)$ is constant in $l$ regions and that WLOG up to a global scale, the minimal amplitude is $a_{\min}=1$. Then for two neighboring regions $R_j$ and $R_{j+1}$ if $p_{j}\geq p_{j+1} \Rightarrow \frac{a_{j+1}}{a_j} \leq \sqrt{\frac{p_{j}}{p_{j+1}}}\leq \sqrt{\frac{p_{\max}}{p_{\min}}}$
and if $p_{j+1}\geq p_j$ $\Rightarrow \frac{a_{j}}{a_{j+1}}\geq 1 \Rightarrow \frac{a_{j+1}}{a_{j}}\leq 1\leq \sqrt{\frac{p_{\max}}{p_{\min}}}$. As a result in each transition between two regions we have
\begin{equation*}
    \frac{a_{i+1}}{a_i}\leq \sqrt{\frac{p_{\max}}{p_{\min}}}
\end{equation*}

Starting from a minimal amplitude of magnitude $1$. For $l$ regions there are no more than $l$ transitions so each amplitude is (loosely) bounded as follows
\begin{equation*}
    a_j \leq a_{min} \left(\sqrt{\frac{p_{\max}}{p_{\min}}}\right)^l = \left(\frac{p_{\max}}{p_{\min}}\right)^\frac{l}{2}
\end{equation*}  
Next we bound the global scale factor. Let $s=\int_{-\pi}^\pi (f(x))^2dx$. Then we have that after normalizing the global scale factor
\begin{equation*}
    a_j \leq \frac{1}{\sqrt{s}}\left(\frac{p_{\max}}{p_{\min}}\right)^\frac{l}{2}
\end{equation*} 
To simplify notation we denote the frequency of each region by $q_j=\frac{\sqrt{p_j}q}{Z}$. Then for $s$ we have: 
\begin{equation*}
   s = \int_{-\pi}^\pi (f(x))^2dx=\sum_{j=1}^la_j^2 \int_{R_j}\cos^2(q_jx+b_{j})dx \geq \sum_{j=1}^la_{\min}^2 \int_{R_j}\cos^2(q_jx+b_{j})dx = \sum_{j=1}^l \int_{R_j}\cos^2(q_jx+b_{j})dx
\end{equation*}
For each region we have
\begin{equation*}
    \int_{R_j}\cos^2(q_jx+b_{j})dx = \int_{-\pi+\frac{2\pi}{l}(j-1)}^{-\pi+\frac{2\pi}{l}j}\cos^2(q_jx+b_{j})dx = 
\end{equation*}
\begin{equation*}
    \frac{1}{2}\int_{-\pi+\frac{2\pi}{l}(j-1)}^{-\pi+\frac{2\pi}{l}j}(1+\cos{}(2q_jx+2b_{j})) dx = \frac{1}{2}\left(x+\frac{\sin(2q_jx+2b_{j})}{2q_j}\right)_{-\pi+\frac{2\pi}{l}(j-1)}^{-\pi+\frac{2\pi}{l}j}=
\end{equation*}
\begin{equation*}
    \frac{1}{2} \left(  -\pi+\frac{2\pi}{l}j+ \frac{\sin(2q_j(-\pi+\frac{2\pi}{l}j)+2b_{j})}{2q_j}-(-\pi+\frac{2\pi}{l}(j-1))-\frac{\sin(2q_j(-\pi+\frac{2\pi}{l}(j-1))+2b_{j})}{2q_j}\right)=
\end{equation*}
\begin{equation*}
    \frac{1}{2}\left(\frac{2\pi}{l}+\frac{\sin(2q_j(-\pi+\frac{2\pi}{l}j)+2b_{j})}{2q_j}-\frac{\sin(2q_j(-\pi+\frac{2\pi}{l}(j-1))+2b_{j})}{2q_j}\right)\geq\frac{\pi}{l}-\frac{1}{2q_j}
\end{equation*}
So we get $s\geq\sum_{j=1}^l \frac{\pi}{l}-\frac{1}{2q_j}=\pi-\frac{1}{2}\sum_{j=1}^l\frac{1}{q_j}=\pi-\frac{1}{2}\sum_{j=1}^l\frac{Z}{\sqrt{p_j}q}$.

And we get:
\begin{equation*}
    s\geq\pi-\frac{1}{2}\sum_{j=1}^l\frac{Z}{\sqrt{p_j}q} = \pi-\frac{Z}{2q}\sum_{j=1}^l\frac{1}{\sqrt{p_j}}
\end{equation*}
As a result all the amplitudes in an eigenfunction of order $q$ are bounded by
\begin{equation}  \label{eq:bound_amp}
    a_i \leq \frac{1}{\sqrt{\pi-\frac{Z}{2q}\sum_{j=1}^l\frac{1}{\sqrt{p_j}}}}\left(\frac{p_{\max}}{p_{\min}}\right)^\frac{l}{2}  ~~~\mathrm{for~all~} 1 \le i \le l
\end{equation}

\section{Local convergence rate as a function of frequency}

To derive the rate of convergence as a function of frequency and density we assume that $p(x)$ forms a piecewise-constant distribution (PCD) with a fixed number of pieces $l$ of equal sizes, $p(x)=p_j$ in $R_j$, ${1 \le j \le l}$. Our proof will rely on a lemma that states informally that not too many eigenfunctions need to be taken into account for convergence -- more precisely, only a number linear in $k$ and inversely linear in $\sqrt{p^*}$, where $p^*>0$ denotes the minimal density. Convergence rate is then determined by the eigenfunction with highest eigenvalue included in the approximation for $g(x)$.

\begin{lemma}  \label{lemma:approx}
Let $p(x)$ be PCD. For any $\epsilon > 0$, there exist $n_k$ such that $\sum_{j={n_k+1}}^\infty g_i^2 < \epsilon^2$, where $g_i=\int_{-\pi}^\pi v_i(x)g(x)p(x)dx$ and $n_k$ is bound as in \eqref{eq:nk_bound} below.
\end{lemma}

\begin{proof}
Given a target function $g(x)=\cos(kx)$ and a basis function $v_i(x) = a(x) \cos(\frac{q_i\sqrt{p(x)}x}{Z}+b(x))$ where $q_i=\lfloor i/2 \rfloor$. (We will assume $a=1$ for now.) Their inner product can be written as
\begin{equation}  \label{eq:inn1}
    g_i = \sum_{j=1}^l a_j p_j \int_{R_j} \cos(kx) \cos(q_{ij}x+b_j)dx
\end{equation}
where $q_{ij}=q_i\sqrt{p_j}/Z$ denotes the local frequency of $v_i(x)$ at $R_j$. Next, to derive a bound we will restrict our treatment to $q_{ij} \ge 2k$ (and by that bound $n_k$ from below). With this assumption we obtain
\begin{align*}
    &\left| \int_{R_j} \cos(kx) \cos(q_{ij}x+b_j)dx \right| \le \left| \int_{-\frac{\pi}{l}}^\frac{\pi}{l} \cos(kx) \cos(q_{ij}x)dx \right| = \nonumber \\ &
    \left| \dfrac{\sin\left(\frac{\pi\left(q_{ij}+k\right)}{l}\right)}{q_{ij}+k}+\dfrac{\sin\left(\frac{\pi\left(q_{ij}-k\right)}{l}\right)}{q_{ij}-k} \right| \le \frac{1}{q_{ij}+k}  + \frac{1}{q_{ij}-k} =  \dfrac{2q_{ij}}{q_{ij}^2-k^2} \le \frac{8}{3q_{ij}}
\end{align*}

Let $p^*=\min_j p_j$ and let $q_i^*=q_i\sqrt{p^*}/Z$, $q_i^*$ denotes the frequency associated with the corresponding region (which is the lowest within $v_i$). Our requirement that $q_{ij} > 2k$ for all $1 \le j \le l$ implies that $q_i^* > 2k$, and therefore
\begin{equation}  \label{eq:app_qi}
    q_i > \frac{2Zk}{\sqrt{p^*}}
\end{equation}
Additionally, using \eqref{eq:inn1}
\begin{equation*}
    |g_i| \le \frac{8}{3} \sum_{j=1}^l \frac{a_j p_j}{q_{ij}}\le \frac{8}{3 q^*_i} \sum_{j=1}^l a_j p_j = \frac{8B}{3 q^*_i} = \frac{8BZ}{3q_i\sqrt{p^*}}
\end{equation*}    
where we denote by $B=\sum_{j=1}^l a_j p_j$ and the equality on the right is obtained by plugging in the definition of $q_i^*$. Note that $\sum_{j=1}^l p_j = l/(2\pi)$ (since $1=\int_{-\pi}^\pi p(x) dx = \sum_{j=1}^l 2\pi p_j/l$), implying that $B \le la^*/(2\pi)$, where $a^*=\max_j a_j$ and $a^*$ is bounded by \eqref{eq:bound_amp}.

Next, for a given $\epsilon>0$ we wish to bound the sum $\sum_{i=n_k}^\infty g_i^2$ by starting from a sufficiently high index $n_k$, i.e.,
\begin{equation*}
    \sum_{i={n_k+1}}^\infty g_i^2 \le \left( \frac{8 B Z}{3 \sqrt{p^*}} \right)^2 \sum_{i={n_k+1}}^\infty \frac{1}{q_i^2} < \frac{1}{q_{n_k}} \left( \frac{8 B Z}{3 \sqrt{p^*}} \right)^2 < \epsilon^2
\end{equation*}
By the definition of $q_i$, $n_k \ge 2 q_{n_k}$, so
\begin{equation*}
    n_k > \frac{2}{\epsilon^2} \left( \frac{8 B Z}{3 \sqrt{p^*}} \right)^2 = \frac{128 B^2 Z^2}{9 \epsilon^2 p^*}
\end{equation*}
So in conclusion,
\begin{equation}  \label{eq:nk_bound}
    n_k > \max \left\{ \frac{4Zk}{\sqrt{p^*}},\, \frac{128 B^2 Z^2}{9 \epsilon^2 p^*} \right\}
\end{equation}
\end{proof}

\begin{theorem}
Let $p(x)$ be a PCD, for any $\delta>0$ the number of iterations $t$ needed to achieve $\|g(x)-u^{(t)}(x)\|<\delta$ is $\tilde O(k^2/p^*)$, where $\tilde O$ hides logarithmic terms.
\end{theorem}

\begin{proof}
Let $n_k$ be chosen as in Lemma~\ref{lemma:approx} with $\epsilon = \delta/2$, i.e.
\begin{equation*}
    n_k = \max \left\{ \frac{4Zk}{\sqrt{p^*}},\, \frac{256 B^2 Z^2}{9 \delta^2 p^*} \right\}
\end{equation*}
Let
\begin{equation*}
    \hat g(x) = \sum_{i=1}^{n_k} g_i v(i)
\end{equation*}
Then,
\begin{equation*}
    \|g(x) - \hat g(x)\|^2 = \sum_{i=n_k+1}^\infty g_i^2 < \left(\frac{\delta}{2}\right)^2
\end{equation*}
and due to triangle inequality
\begin{equation*}
    \|g(x)-u^{(t)}(x)\| \le \|g(x)-\hat g(x)\| + \|\hat g(x)-u^{(t)}(x)\|
\end{equation*}
it suffices to find $t$ such that
\begin{equation*}
     \|\hat g(x)-u^{(t)}(x)\| < \frac{\delta}{2} = \tilde \delta
\end{equation*}
Using (Arora et al., 2019b)'s Theorem 4.1 adapted to continuous operators
\begin{equation}  \label{eq:delta}
    \Delta^2 = \|\hat g - u^{(t)}\|^2 \approx \sum_{i=1}^{n_k} (1-\eta \lambda_i)^{2t} g_i^2 \le \pi \sum_{i=1}^{n_k} (1-\eta \lambda_i)^{2t} \le \pi n_k (1 - \eta \lambda_{n_k})^{2t}
\end{equation}
where the left inequality is due to $|g_i|^2 \le \| \cos^2(kx) \| = \pi$ and the right inequality is because $\lambda_i$ are arranged in a descending order. Now for a fixed distribution $p(x)$, and since we are interested in the asymptotic rate of convergence (i.e., as $k \rightarrow \infty$), as soon as $k > 64 B^2 Z / (9 \tilde\delta^2 \sqrt{p^*})$ it suffices to only consider the case $q_{n_k} = 2Zk/\sqrt{p^*}$, as in \eqref{eq:app_qi}. The eigenvalue $\lambda_{n_k}$ is determined according to
\begin{equation*}
    \lambda_{n_k} = \frac{Z^2}{\pi^2 q_{n_k}^2} = \frac{p^*}{4\pi^2 k^2}
\end{equation*}
(Here we used the expression for $\lambda_{n_k}$ assuming $n_k$ is odd. A similar expression of the same order is obtained for even $n_k$.) Consequently, to bound $\Delta^2 < \tilde\delta$ in \eqref{eq:delta} and substituting for $n_k$ and $\lambda_{n_k}$ we have
\begin{equation*}
    \frac{4Zk}{\sqrt{p^*}} \left(1 - \frac{\eta p^*}{4\pi^2 k^2} \right)^{2t} < \tilde\delta
\end{equation*}
Taking log
\begin{equation*}
    2t \log\left(1 - \frac{\eta p^*}{4\pi^2 k^2} \right) > \log\left( \frac{\delta\sqrt{p^*}}{4Zk}\right)
\end{equation*}
from which we obtain
\begin{equation*}
    t > \frac{\log \left( \frac{\delta\sqrt{p^*}}{4Zk}\right)}{2\log\left(1 - \frac{\eta p^*}{4\pi^2 k^2} \right)} \approx -\frac{2\pi^2 k^2}{\eta p^*} \log\left(\frac{\delta\sqrt{p^*}}{4Zk}\right) = \tilde O \left(\frac{k^2}{p^*} \right)
\end{equation*}
where $\tilde O$ hides logarithmic terms.

\end{proof}

\comment{
-------------

Maybe this can be useful..

\begin{equation}
    A = A(k,d,l) = \int_{-\frac{\pi}{l}}^{\frac{\pi}{l}} \cos(kx) \cos((k+d)x)dx = \dfrac{\sin\left(\frac{{\pi}\left(2k+d\right)}{l}\right)}{2k+d}+\dfrac{\sin\left(\frac{{\pi}d}{l}\right)}{d}
\end{equation}
\begin{equation}
    \sin\left(\frac{{\pi}\left(2k+d\right)}{l}\right) = \sin\left(\frac{2\pi k}{l}\right) \cos\left(\frac{\pi d}{l}\right) + \cos\left(\frac{2\pi k}{l}\right) \sin\left(\frac{\pi d}{l}\right)
\end{equation}
Assuming $k/l$ is integer then $\sin\left(\frac{2\pi k}{l}\right)=0$ and $\cos\left(\frac{2\pi k}{l}\right)=1$, so
\begin{equation}
    \sin\left(\frac{{\pi}\left(2k+d\right)}{l}\right) =   \sin\left(\frac{\pi d}{l}\right)
\end{equation}
Therefore,
\begin{equation}
    A = \left( \dfrac{1}{2k+d}+\dfrac{1}{d} \right) \sin\left(\frac{{\pi}d}{l}\right) = \dfrac{2(k+d)}{d(2k+d)}\sin\left(\frac{{\pi}d}{l}\right) = \left(1 + \dfrac{d}{2k+d} \right) \dfrac{1}{d} \sin\left(\frac{{\pi}d}{l}\right)
\end{equation}
Assuming $|d| \le 0.5$, we note that
\begin{equation}
    \dfrac{d}{2k+d} \lessapprox \dfrac{1}{4k} 
\end{equation}
and
\begin{equation}
    0.9 \dfrac{\pi}{l} \le \dfrac{1}{d} \sin\left(\frac{{\pi}d}{l}\right) \le \dfrac{\pi}{l}
\end{equation}

In principle in every region $R_j$ we should have $|d_j| \le \sqrt{p_j}/(2Z)$. Since
\begin{equation}
    Z = \frac{1}{2\pi}\int_{-\pi}^\pi \sqrt{p(x)} = \frac{1}{l} \sum_{j=1}^l \sqrt{p_j}
\end{equation}
we obtain that
\begin{equation}
    |d_j| \le \frac{l\sqrt{p_j}}{2 \sum_{j=1}^l \sqrt{p_j}}    
\end{equation}

}

\onecolumn

\section{Spectral convergence analysis  for deep networks - proof of Theorem 2}

\subsection{The network model}

The parameters of the network are $W = (W_1, ..., W_L)$ where $W_l \in \Real^{m \times m}$ and also $A \in \Real^{m \times d}$ and $B \in \Real^{1 \times m}$. The network function over  input $\x_i \in \Real^d$ ($i \in \left[ n \right]$) is given by   
\[ u_i =f(\x_i; W) = B \sigma (W_L\sigma (W_{L-1}\sigma(....(W_1\sigma(Ax_i))..)) \]
where $\sigma$ stands for element wise RELU activation function. For a tuple $W = (W_1, ..., W_L)$ of matrices, we let $\norm{W}_2 = \max_{l \in [L]} \norm{W_l}_2$ and $\norm{W}_F = (\sum_{l=1}^L \norm{W_l}_F^2)^{1/2}$.

The parameters are initialized randomly from a normal distribution according to  
\begin{align}
\label{eq:random_init}
[W_l]_{ij} &\sim \mathcal{N}(0,\frac{2}{m}), \, l \in [L] \\ \nonumber
A_{ij} &\sim \mathcal{N}(0,\frac{2}{m}) \\ \nonumber
B_{ij} &\sim \mathcal{N}(0,\tau^2) 
\end{align}
where similarly to   \cite{allenzhu} the layers $A$ and $B$ are initialized and held fixed. 

The network functionality is summarized as follows 
\begin{align*}
\h_{i,0} &= \sigma(A\x_i) \\
\h_{i,l}^{(t)} &= \sigma(W_l^{(t)}\h_{i,l-1}^{(t)}) \\
\uu_i^{(t)} &= B\h_{i,L}^{(t)}
\end{align*}
where $i \in \left[n\right]$, $l \in \left[L\right]$ and $t$ denotes iteration number. In addition, for each input vector $i \in \left[n\right]$ and layer $l \in \{ 0,1, ..., L \}$, we associate  a diagonal matrix $D_{i,l}$ such that for $j \in \left[ m \right]$, $(D_{i,l})_{j,j}=\mathbb{I}_{(W_l \h_{i,l-1})_j \geq 0}$, where we use the convention $\h_{i,-1}=\x_i$.
The network is trained to minimize the $\ell_2$ loss 
\[\Phi(W)= \frac{1}{2}\sum_{i=1}^{n} (y_i-f(\x_i; W))^2 \] 
 

We will analyze the properties of the  matrices $H, H^{\infty} \in \mathbb{R}^{n \times n}$, comprised of the following entries
\[ H_{ij}(t)=\left \langle \frac{\partial u_i^{(t)}}{\partial W},\frac{\partial u_j^{(t)}}{\partial W} \right \rangle \]

\[ H_{ij}^{\infty} = \mathbb{E}_{W} \left\langle \frac{\partial u_i^{(0)}}{\partial W}, \frac{\partial u_j^{(0)}}{\partial W} \right\rangle. \]

We write the eigen-decomposition of $H^{\infty}=\sum_{i=1}^n \lambda_i \vv_i \vv_i^T$, where $\vv_1, \ldots, \vv_n$ are the eigenvectors of $H^{\infty}$ and $\lambda_1, \ldots, \lambda_n$ are their corresponding eigenvalues. The minimal eigenvalue is denoted by   $\lambda_0=\min(\lambda(H^\infty))$.

\begin{theorem} \label{Thm:main}  
For any $\epsilon \in (0,1]$ and  $\delta \in (0,O(\frac{1}{L})]$, let $\tau=\Theta(\frac{\epsilon\hat{\delta}}{n})$, $m \geq \Omega \left( \frac{n^{24} L^{12} \log^5 m}{\delta^8 \tau^6} \right)$, $\eta = \Theta \left( \frac{\delta}{n^4 L^2 m \tau^2}\right)$. Then, with probability of at least $1-\hat{\delta}$ over the random initialization after $t$ iterations of GD we have that \begin{equation}
\label{eq:finegrained}
    \|\y-\uu^{(t)}\| = \sqrt{\sum_{i=1}^n (1-\eta \lambda_i)^{2t}(\vv_i^T \y)^2} \, \pm \epsilon.
\end{equation}
\end{theorem}

\subsection{Proof strategy}

The proof of Thm. \ref{Thm:main} relies on a theorem, provided by \cite{allenzhu}, stated in Thm.\ \ref{Thm_conv_Zhu}, and an observation, based  the on the derivation of the proof to that theorem, which we state in Lemma \ref{Allen-Zhu-corollary}.

Thm.\ \ref{Thm_conv_Zhu} assumes that the data is normalized, so that $\norm{\x_i}=1$, and there exists $\delta \in (0,O(\frac{1}{L})]$ such that for every pair $i, j \in [n]$, we have $\norm{\x_i - \x_j} \geq \delta$ and also it holds that $\abs{y_i} \leq O(1)$.

In addition, we prove  Lemma \ref{mainThm}, which is the basis for the proof of our Theorem. 

\begin{lemma}\label{mainThm}
Suppose  $\delta\in(0,O(\frac{1}{L})] $, $m \geq \Omega \left( \frac{n^{24} L^{12} \log^5 m}{\delta^8 \tau^2 } \right)$, $\eta = \Theta \left( \frac{\delta}{n^4 L^2 m \tau^2}\right)$ and also let $\omega =O(\frac{n^3 \log m}{\delta \tau \sqrt{m}}) $. Then, with probability at least $1-e^{-\Omega(m\omega^{2/3}L)} $ over the randomness of  $A, B$ and $W^{(0)}$ we have  
\begin{align}\label{epsilon_eq}
\uu(t+1)-\y=(I-\eta H(t))(\uu(t)-\y)+\epsilon(t)
\end{align}
with 
$$ \norm{\epsilon(t)} \leq O\left(\frac{L \log^{4/3} m}{\tau ^{1/3} m^{1/6}n^{1.5}}\right)\sqrt{\Phi(W^{(t)})} + O\left(\frac{\delta^2}{\tau n^6m^{0.5} L^{1.5}}\right)\Phi(W^{(t)})$$ 
\end{lemma}
The proof of the Lemma is deferred, and will be given after the proof of the theorem.  

\subsection{Proof of Thm \ref{Thm:main}}

\begin{proof}
By Lemma \ref{mainThm} we have the following relation 
\begin{equation*}
\uu(t)-\y = (I-\eta H(t-1))(\uu(t-1)-\y)+\epsilon(t-1)
\end{equation*}
Adding and subtracting $\eta H^\infty(\uu(t-1)-\y)$ we have
\begin{equation*}
\uu(t)-\y = (I-\eta H^\infty)(\uu(t-1)-\y)+\eta( H^\infty-H(t-1))(\uu(t-1)-\y)+\epsilon(t-1)
\end{equation*}
and this is equivalent to 
\begin{equation}\label{eq:rec}
\uu(t)-\y = (I-\eta H^\infty)(\uu(t-1)-\y)+\xi(t-1).
\end{equation}
where we denote $\xi(t)=\eta( H^\infty-H(t))(\uu(t)-\y)+\epsilon(t)$. Then, by applying \eqref{eq:rec} recursively, we obtain 
\begin{equation}\label{eq:main}
\uu(t)-\y =
(I-\eta H^\infty)^t(\uu(0)-\y)+\sum_{i=0}^{t-1}(I-\eta H^\infty)^i \xi(t-1-i)\\
\end{equation}

We first bound the quantity $\norm{\xi(t-1-i)}_2$
\begin{align*}
&\norm{\xi(t-1-i)}_2=\norm{\eta(H(t-1-i)-H^\infty) ( y-u(t-1-i))+\epsilon(t-1-i)}_2\\
&\leq \norm{\eta(H(t-1-i)-H^\infty)}_2\norm{( y-u(t-1-i))}_2+\norm{\epsilon(t-1-i)}_2 \\
& \eta\leq^{^{1,2}}{O\left(\frac{\delta^2 m\tau^3}{n^6}\right)}\sqrt{\Phi(W^{(t-1-i)})})+ O\left(\frac{\delta^2}{\tau n^6m^{0.5}L^{1.5}}\right)\Phi(W^{(t-1-i)})+O\left(\frac{L \log^{4/3} m}{\tau ^{1/3}m^{1/6}n^{1.5}}\right)\sqrt{\Phi(W^{(t-1-i)})}\\
&\leq^{^3} \left(1-\Omega\left(\frac{\tau^2\eta\delta m}{n^2}\right)\right)^{\frac{t-1-i}{2}}\left({\eta O\left(\frac{\delta^2 m\tau^3}{n^6}\right)}\sqrt{\Phi(W^{(0)}})+ O\left(\frac{\delta^2}{\tau n^6m^{0.5}L^{1.5}}\right)\Phi(W^{(0)})+O\left(\frac{L \log^{4/3} m}{\tau ^{1/3}m^{1/6}n^{1.5}}\right)\sqrt{\Phi(W^{(0)})}\right) \\
&\leq^{^4} \left(1-\Omega\left(\frac{\tau^2\eta\delta m}{n^2}\right)\right)^{\frac{t-1-i}{2}}\left(\eta O\left(\sqrt{n}\right){O\left(\frac{\delta^2 m\tau^3}{n^6}\right)}+O\left(\frac{\delta^2}{\tau n^6m^{0.5}L^{1.5}}\right)^{}O\left(n\right) + O\left(\frac{L \log^{4/3} m}{\tau ^{1/3}m^{1/6}n^{1.5}}\right)O\left(\sqrt{n }\right)\right)\\
& = \left(1-\Omega\left(\frac{\tau^2 \eta\delta m}{n^2}\right)\right)^{\frac{(t-1-i)}{2}}\left({\eta O\left(\frac{\delta^2 m\tau^3}{n^{5.5}}\right)}+O\left(\frac{\delta^2}{\tau n^5m^{0.5}L^{1.5}}\right)+O\left(\frac{L \log^{4/3} m}{\tau ^{1/3}m^{1/6}n}\right)\right)
\end{align*}
where we make the following derivations
\begin{enumerate}
\item Using Lemma \ref{corollary_Hinf} which states that $\norm{H(t)-H^\infty}_2 \leq O(\frac{\delta^2 m\tau^3}{n^6})$.
\item Using the bound in Lemma \ref{mainThm}, for $\epsilon(t-1-i)$
\item Using bound over the loss by, Lemma  \ref{Allen-Zhu-corollary} (b).
\item By Lemma \ref{init_bound-lemma} the loss at initialization is bounded by $O(n)$. 
\end{enumerate}
Using the bound, derived above, \eqref{eq:main} yields 
\begin{align*}
& \norm{\uu(t)-\y} = \norm{(I-\eta H^\infty)^t(\uu(0)-\y)+ \sum_{i=0}^{t-1}((I-\eta H^\infty)^i \xi(t-1-i))}\\
&\leq^{^{1}} \norm{(I-\eta H^\infty)^t(\uu(0)-\y)} \\ 
&+
\sum_{i=0}^{t-1}(1-\eta \lambda_0)^i\left(1-\Omega\left(\frac{\tau^2 \eta\delta m}{n^2}\right)\right)^{\frac{(t-1-i)}{2}}\left({\eta O\left(\frac{\delta^2 m\tau^3}{n^{5.5}}\right)}+O\left(\frac{\delta^2}{\tau n^5m^{0.5}L^{1.5}}\right)+O\left(\frac{L \log^{4/3} m}{\tau ^{1/3}m^{1/6}n}\right)\right)\\
&\leq^{^{2}}
\norm{(I-\eta H^\infty)^t(\uu(0)-\y)} + t\left({\eta O\left(\frac{\delta^2 m\tau^3}{n^{5.5}}\right)}+O\left(\frac{\delta^2}{\tau n^5m^{0.5}L^{1.5}}\right)+O\left(\frac{L \log^{4/3} m}{\tau ^{1/3}m^{1/6}n}\right)\right) \\
&\leq^{^{3}} \norm{(I-\eta H^\infty)^t(\uu(0)-\y)}+O\left(\frac{n^6 L^2}{\delta^2}\right) \left({\eta O\left(\frac{\delta^2 m\tau^3}{n^{5.5}}\right)}+O\left(\frac{\delta^2}{\tau n^5m^{0.5}L^{1.5}}\right)^{}+O\left(\frac{L \log^{4/3} m}{\tau ^{1/3}m^{1/6}n}\right) \right)\\
&\leq \norm{(I-\eta H^\infty)^t} \norm{\uu(0)} + \norm{(I-\eta H^\infty)^t \y} + O\left(\frac{n^6 L^2}{\delta^2}\right)\left({\eta O\left(\frac{\delta^2 m\tau^3}{n^{5.5}}\right)}+O\left(\frac{\delta^2}{\tau n^5m^{0.5}L^{1.5}}\right)^{}+O\left(\frac{L \log^{4/3} m}{\tau ^{1/3}m^{1/6}n}\right) \right)
\end{align*}
where we make the following derivations 
\begin{enumerate}

\item $\norm{I - \eta H^{\infty}}_2$ is bounded by the maximal eigenvalue of the positive definite matrix $(I - \eta H^{\infty})$, i.e, $(1 - \eta \lambda_0)$.

\item  $(1-\eta \lambda_0)^i\left(1-\Omega\left(\frac{\tau^2 \eta\delta m}{n^2}\right)\right)^{\frac{(t-1-i)}{2}} \leq 1$ 

\item By Theorem \ref{Thm_conv_Zhu}, $t\leq O(\frac{n^6 L^2}{\delta^2})$

\end{enumerate}


Next, it is straightforward to show that 
\begin{equation}\label{eq:1}
\norm{(I-\eta H^\infty)^t \y}=\sqrt{\sum_{i=1}^n(1-\eta \lambda_i)^{2t}(\vv_i^T\y)^2}
\end{equation}
where $\lambda_i, \vv_i$ are the eigenvalues and  eigenvectors of $H^\infty$, respectively. 

For the first term we use lemma \ref{init_bound-lemma} which states that $\norm{\uu(0)}\leq\frac{\sqrt{n} \tau}{\hat{\delta}}$, and by our choice of $\tau$ we obtain 
\begin{equation}\label{eq:2} \norm{(I-\eta H^\infty)^t} \norm{\uu(0)} \leq (1-\eta \lambda_0)^{t} O\left(\frac{\sqrt{n}\tau}{\hat{\delta}}\right)\leq \epsilon  \end{equation}
Finally, by our choice of $\eta,m,\tau$ it holds that 
\begin{equation}\label{eq:3} O\left(\frac{n^6 L^2}{\delta^2}\right)\left({O\left(\frac{\delta^2 m\tau^3}{n^{5.5}}\right)}\eta+O\left(\frac{\delta^2}{\tau n^5m^{0.5}L^{1.5}}\right)^{}+O\left(\frac{L \log^{4/3} m}{\tau ^{1/3}m^{1/6}n}\right) \right)\leq \epsilon
\end{equation}
Combining \eqref{eq:1}, \eqref{eq:2} and \eqref{eq:3} yields
\begin{equation}
\norm{\y-\uu(t)}=\sqrt{\sum_{i=1}^n(1-\eta \lambda_i)^{2k}(\vv_i^T\y)^2} \pm {\epsilon} 
\end{equation}
\end{proof}

\subsection{Supporting Lemmas}
\begin{proof} Proof of  Lemma \ref{mainThm}.

By construction
\begin{align*}
\epsilon_i(t)&= u_{_{i}}(t+1)-u_{i}(t)+\left[\eta H(t)(\uu(t)-\y)\right]_i\\
&=u_i(t+1)-u_i(t)+\eta \sum_{j=1}^n(u_{j}(t)-y_{j})H_{ij}(t)\\
&=u_i(t+1)-u_i(t)+\eta \sum_{j=1}^n(u_{j}(t)-y_{j}) \left\langle \frac{\partial u_i(t)}{\partial W},\frac{\partial u_j(t)}{\partial W} \right\rangle\\
&=u_i(t+1)-u_i(t)+\eta \left\langle \frac{\partial u_i(t)}{\partial W},\sum_{j=1}^n(u_{j}(t)-y_{j})\frac{\partial u_j(t)}{\partial W} \right\rangle\\
&=u_i(t+1)-u_i(t)+\eta \left\langle \frac{\partial u_i}{\partial W},\nabla \Phi(W^{(t)}) \right\rangle.\\
\end{align*}
We denote  $-\eta\nabla \Phi(W^{(t)})$ by $W'=(W_1^{'}, ..., W_L^{'})$, yielding
\begin{align*}
\epsilon_i(t) &=u_i{(t+1)}-u_i{(t)}-\left\langle \frac{\partial u_i (t)}{\partial W},W' \right\rangle\\
&=B(h_{i,L}^{(t+1)}-h_{i,L}^{(t)}) -\left\langle \frac{\partial u_i(t)}{\partial W},W' \right\rangle\\
&=B(h_{i,L}^{(t+1)}-h_{i,L}^{(t)}-\sum_{l=1}^L D_{i,L}^{(t)}W_L^{(t)}  D_{i,L-1}^{(t)}W_{L-1}^{(t)} \cdots D_{i,L+1}^{(t)} W_{l+1}^{(t)}D_{i,l}^{(t)}W_l'h_{i,l-1}^{(t)}) \\
& = B \left( \sum_{l=1}^L (D_{i,L}^{(t)}+D_{i,L}'')W_L^{(t)} \cdots W_{l+1}^{(t)}(D_{i,l}^{(t)}+D_{i,l}'')W_l'h_{i,l-1}^{(t+1)}-\sum_{l=1}^L D_{i,L}^{(t)}W_L^{(t)} \cdots W_{l+1}^{(t)}D_{i,l}^{(t)}W_l'h_{i,l-1}^{(t)} \right)
\end{align*}
where the last equality is obtained by replacing $h_{i,L}^{(t+1)} - h_{i,L}^{(t)}$ by the term provided in Lemma \ref{zho_11.2}, where $D_{i,l}'' \in \Real^{m \times m}$ are diagonal matrices with entries in $[-1,1]$.

Now, we derive a bound for $\abs{\epsilon_i(t)}$.
We start by subtracting and adding the same term, yielding 
\begin{align*}
\abs{\epsilon_i(t)} &= | B (\sum_{l=1}^L (D_{i,L}^{(t)}+D_{i,L}'')W_L^{(t)} \cdots W_{l+1}^{(t)}(D_{i,l}^{(t)}+D_{i,l}'')W_l'h_{i,l-1}^{(t+1)}-D_{i,L}^{(t)}W_L^{(t)} \cdots W_{l+1}^{(t)}D_{i,l}^{(t)}W_l'h_{i,l-1}^{(t+1)} \\ 
&+\sum_{l=1}^L D_{i,L}^{(t)}W_L^{(t)} \cdots W_{l+1}^{(t)}D_{i,l}^{(t)}W_l'h_{i,l-1}^{(t+1)} - D_{i,L}^{(t)}W_L^{(t)} \cdots W_{l+1}^{(t)}D_{i,l}^{(t)}W_l'h_{i,l-1}^{(t)} ) | \\
&\leq \sum_{l=1}^L \abs{B\left((D_{i,L}^{(t)}+D_{i,L}'')W_L^{(t)}...W_{l+1}^{(t)}(D_{i,l}^{(t)}+D_{i,l}'')W_l'h_{i,l-1}^{(t+1)}-D_{i,L}^{(t)}W_L^{(t)} \cdots W_{l+1}^{(t)}D_{i,l}^{(t)}W_l'h_{i,l-1}^{(t+1)}\right)}\\ &+\sum_{l=1}^L \abs{B \left(D_{i,L}^{(t)}W_L^{(t)} \cdots W_{l+1}^{(t)}D_{i,l}^{(t)}W_l'h_{i,l-1}^{(t+1)}-D_{i,L}^{(t)}W_L^{(t)} \cdots W_{l+1}^{(t)}D_{i,l}^{(t)}W_l'h_{i,l-1}^{(t)}\right)}.
\end{align*}
To construct the bound for $\abs{\epsilon_i(t)}$, we separately bound each of the above two terms. For the first term
\begin{align*}
    &\abs{B\left((D_{i,L}^{(t)}+D_{i,L}'')W_L^{(t)}...W_{l+1}^{(t)}(D_{i,l}^{(t)}+D_{i,l}'')W_l'h_{i,l-1}^{(t+1)}-D_{i,L}^{(t)}W_L^{(t)}...W_{l+1}^{(t)}D_{i,l}^{(t)}W_l'h_{i,l-1}^{(t+1)}\right)} \\
    &\leq \norm{B \left((D_{i,L}^{(t)}+D_{i,L}'')W_L^{(t)}...W_{l+1}^{(t)}(D_{i,l}^{(t)}+D_{i,l}'')-D_{i,L}^{(t)}W_L^{(t)}...W_{l+1}^{(t)}D_{i,l}^{(t)}\right)}_2 \norm{W_l'h_{i,l-1}^{(t+1)}}_2 \\
    &\leq^{^1}  \norm{B \left((D_{i,L}^{(t)}+D_{i,L}'')W_L^{(t)}...W_{l+1}^{(t)}(D_{i,l}^{(t)}+D_{i,l}'')-D_{i,L}^{(0)}W_L^{(0)}...W_{l+1}^{(0)}D_{i,l}^{(0)}\right)}_2 O(\norm{W'_l}_2)\\ &+\norm{B\left(D_{i,L}^{(0)}W_L^{(0)}...W_{l+1}^{(0)}D_{i,l}^{(0)}-D_{i,L}^{(t)}W_L^{(t)}...W_{l+1}^{(t)}D_{i,l}^{(t)}\right)}_2 O(\norm{W'_l}_2)) \\
    &=^{^2} \norm{B \left(D_{i,L}^{(0)}-D_{i,L}^{(0)}+D_{i,L}^{(t)}+D_{i,L}'')W_L^{(t)}...W_{l+1}^{(t)}(D_{i,l}^{(0)}-D_{i,l}^{(0)}+D_{i,l}^{(t)}+D_{i,l}'')-D_{i,L}^{(0)}W_L^{(0)}...W_{l+1}^{(0)}D_{i,l}^{(0)}\right)}_2 O(\norm{W'_l}_2)\\ 
    &+\norm{B\left(D_{i,L}^{(0)}W_L^{(0)}...W_{l+1}^{(0)}D_{i,l}^{(0)}-(D_{i,L}^{(0)}-D_{i,L}^{(0)}+D_{i,L}^{(t)})W_L^{(t)}...W_{l+1}^{(t)}(D_{i,l}^{(0)}-D_{i,l}^{(0)}+D_{i,l}^{(t)})\right) }_2  O(\norm{W'_l}_2)\\
    &\leq^{^3} O(\tau \omega^{1/3}L^2 \sqrt{m \log m})O(\norm{W'_l}_2)
\end{align*}
where we apply the following derivations
\begin{enumerate}
    \item We subtract and add the same term, use triangle inequality and the result provided in Lemma \ref{Allen-Zhu_7.1}, $\norm{h_{i,l-1}^{(t+1)}} = O(1)$.
    \item Subtract and add $D_{i,l}^{(0)}$ from each coefficient that multiply $W_l^{(t)}$.
    \item Due to Lemma \ref{Allen-Zhu-corollary}, it holds that $||W^{(t)}-W^{(0)}||\leq \omega$. This enables us to use Lemma \ref{lemma8.2_zhu}, implying that $\|D_{i,l}^{(t)}-D_{i,l}^{(0)}\|_0\leq s = O(m \omega^{2/3} L)$. Moreover, in conjunction with  Lemma \ref{zho_11.2}, this yields  $\norm{D_{i,l}^{(t)}+D_{i,l}''-D_{i,l}^{(0)}}_0 \leq s$. Having that, we can apply Lemma \ref{zhu_8.7}, to obtain a bound for the first term.
\end{enumerate}

For the second term we have that:
\begin{align*}
& \abs{B(D_{i,L}^{(t)}W_L^{(t)}...W_{l+1}^{(t)}D_{i,l}^{(t)}W_l'h_{i,l-1}^{(t+1)}-D_{i,L}^{(t)}W_L^{(t)}...W_{l+1}^{(t)}D_{i,l}^{(t)}W_l'h_{i,l-1}^{(t)})}\\
&=\abs{B (D_{i,L}^{(t)}W_L^{(t)}...W_{l+1}^{(t)}D_{i,l}^{(t)}W_l'(h_{i,l-1}^{(t+1)}-h_{i,l-1}^{(t)}))}\\
&\leq\left(\norm{B (D_{i,L}^{(t)}W_L^{(t)}...W_{l+1}^{(t)}D_{i,l}^{(t)}- D_{i,L}^{(0)}W_L^{(0)}...W_{l+1}^{(0)}D_{i,l}^{(0)})}+\norm{B D_{i,L}^{(0)}W_L^{(0)}...W_{l+1}^{(0)}D_{i,l}^{(0)}} \right)\norm{W'_l} \norm{h_{i,l-1}^{(t+1)}-h_{i,l-1}^{(t)}}\\
&\leq^{^1} \left( O(\tau \omega^{1/3} L^2 \sqrt{m \log m})+\norm{B D_{i,L}^{(0)}W_L^{(0)}...W_{l+1}^{(0)}D_{i,l}^{(0)}} \right)\norm{W'_l} \norm{h_{i,l-1}^{(t+1)}-h_{i,l-1}^{(t)}}\\
&\leq^{^2} \tau O(\sqrt{m}+\omega^{1/3}L^2\sqrt{m \log m}) \norm{W'_l} \norm{h_{i,l-1}^{(t+1)}-h_{i,l-1}^{(t)}}
\leq^{^3} \tau O(\sqrt{m}+\omega^{1/3}L^2\sqrt{m \log m}) L^{1.5}\norm{W'}^2 \\
&\leq^{^4} O(\tau \sqrt{m})L^{1.5} \norm{W'}^2 
\end{align*}
where we apply the following derivations
\begin{enumerate}
\item As in the previous derivation, using Lemma \ref{zhu_8.7}.
\item Applying Lemma \ref{zhu_7.4}. 
\item Using Lemma \ref{zho_11.2}.
\item Plug in $\omega = \frac{n^3 \log m}{\delta \tau \sqrt{m}}$.
\end{enumerate}

Since $W'=-\eta \nabla \Phi(W^{(t)})$,  we can get a bound for $\norm{W'}_2$ using Lemma \ref{zhu_Thm3}, yielding $\norm{W'}_2 \leq \eta O(\tau \sqrt{n m} \sqrt{\Phi(W^{(t)})})$.

Taking into account the two bounds, and summing over the all layers and data points we obtain that  
$$\norm{\epsilon(t)} \leq n L O(\tau w^{1/3} L^2 \sqrt{m \log m}) O(\eta \tau \sqrt{nm} \sqrt{\Phi(W^{(t)})}) +
nL O(\tau \sqrt{m})L^{1.5} O(\eta^2 \tau^2 nm \Phi(W^{(t)}))  $$


Using our choice of $\eta$ and the value of $\omega$, we finally get 
\begin{equation*}
\norm{\epsilon(t)} \leq O\left(\frac{L \log^{4/3} m}{\tau ^{1/3} m^{1/6}n^{1.5}}\right)\sqrt{\Phi(W^{(t)})} + O\left(\frac{\delta^2}{\tau n^6m^{0.5} L^{1.5}}\right)\Phi(W^{(t)}) 
\end{equation*}


\end{proof}

\begin{theorem}\label{Thm_conv_Zhu} \footnote{This theorem was proved in \cite{allenzhu}, for $\tau=1$. However, it is straightforward to generalize it for $\tau \in (0,1]$ at the price of modifying $m$ and $\eta$ by a factor of  $\frac{1}{\tau^2}$} For any $\epsilon \in (0,1]$ and  $\delta \in (0,O(\frac{1}{L})]$, let  $m \geq \Omega \left( \frac{n^{24} L^{12} \log^5 m}{\delta^8 \tau^2} \right)$, $\eta = \Theta \left( \frac{\delta}{n^4 L^2 m \tau^2}\right)$ and $W^{(0)}, A, B$ are at random initialization \eqref{eq:random_init}. Then, starting from Gaussian initialization, with probability at least $1-e^{-\Omega(log^2m)}$, gradient descent with learning rate $\eta$ achieves 
\[ \Phi(W)\leq \epsilon ~~\text{in}~~ T=\Theta\left( \frac{n^6 L^2}{\delta^2} \log \frac{1}{\epsilon} \right) \] 
\end{theorem}

\begin{lemma}\label{Allen-Zhu-corollary}
Under the assumptions of Thm. \ref{Thm_conv_Zhu},
it holds that for every $t=0,1,..,T-1$
\begin{align*}
(a)~~~~~& \norm{W^{(t)}-W^{(0)}}_F  \leq\omega:=O \left( \frac{n^3}{\delta \tau \sqrt{m}}\log m \right) \\
(b)~~~~~& \Phi(W^{(t)})  \leq \left(1-\Omega\left(\frac{\tau^2\eta\delta m}{n^2}\right)\right)^{t}\Phi(W^{(0)}) 
\end{align*}
\end{lemma}

\begin{lemma}\label{zho_11.2}(This Lemma follows Claim 11.2 from  \cite{allenzhu})
Let $\omega \in [\Omega(\frac{1}{\tau^3 m ^{3/2} L^{3/2} \log^{3/2}m}),O(\frac{1}{L^{4.5} \log^3 m})]$, then under the following assumptions $\norm{W^{(t)}-W^{(0)}}_2 \leq \omega$ and $\norm{W'}_2 \leq w$ it holds that 
there exist diagonal matrices $D''_{i,l}\in\mathbb{R}^{m\times m}$ with entries in [-1,1] such that
\begin{align*}
\forall i\in[n],\forall l\in[L]: h^{(t+1)}_{i,l}-h^{(t)}_{i,l}=\sum_{a=1}^l(D_{i,l}^{(t)}+D''_{i,l})W_l^{(t)}...W_{a+1}^{(t)}(D_{i,a}^{(t)}+D_{i,a}'')W'_ah^{(t+1)}_{i,a-1}
\end{align*} 
Furthermore we have $\norm{h^{(t+1)}_{i,l}-h^{(t)}_{i,l}}\leq O(L^{1.5}) \norm{W'}_2$ and $\norm{Bh^{(t+1)}_{i,l}-Bh^{(t)}_{i,l}}\leq O(L \tau \sqrt{m}) \norm{W'}_2$ and $\norm{D_{i,l}''}_0\leq O(m\omega^{2/3}L)$
\end{lemma}
\begin{lemma}\label{lemma8.2_zhu}(This Lemma follows Lemma 8.2 from \cite{allenzhu}) Suppose $\omega\leq \frac{1}{CL^{9/2}log^3m}$ for some sufficiently large constant $C>1$. With probability at least $1-e^{-\Omega(m\omega^{2/3}L)}$ for every $(W^{(t)}-W^{(0)})$ satisfying $\norm{W^{(t)}-W^{(0)}}_2\leq \omega$,
$$ \norm{D_{i,l}^{(t)} - D_{i,l}^{(0)}}_0\leq O(m\omega ^{2/3}L)$$ 
\end{lemma}

\begin{lemma}\label{zhu_8.7}(This Lemma follows Lemma 8.7  from \cite{allenzhu}) For $s = O(mw^{2/3} L)$, with probability at least $1-e^{-\Omega(s \log m)}$ over the randomness of $W^{(0)},A,B$
\begin{itemize}
\item for all $i\in [n], a \in [L+1]$
\item for every diagonal matrices $D_{i,0}''', \cdots ,D_{i,L}'''\in [-3,3]^{m\times m}$ with at most s non-zero entries
\item for every perturbation with respect to the initialization  $W''_{1} \cdots W''_{L}\in \mathbb{R}^{m\times m}$ with $\norm{W''}_2\leq \omega=O(1/L^{1.5})$
\end{itemize}
it holds $\norm{ B(D_{i,L}^{(0)}+D'''_{i,L})(W_L^{(0)}+W_L'') \cdots (W_{a+1}^{(0)}+W_{a+1}'')(D_{i,a}^{(0)}+D'''_{i,a})-BD_{i,L}^{(0)}W_L^{(0)} \cdots W_{a+1}^{(0)}D_{i,a}^{(0)}}_2\leq O(\tau \omega^{1/3}L^2\sqrt{m \log m})$
\end{lemma}

\begin{lemma}\label{zhu_7.4}(This Lemma follows Lemma  7.4b from  \cite{allenzhu}) 
 Suppose $m\geq\Omega (nL \log (nL)).$ If $ s = O(m \omega^{2/3} L)$ then with probability at least  $1-e^{-\Omega(s \log m)}$ for all $i\in [n], a \in [L+1]$ it holds that $\norm{v^TBD_{i,L}^{(0)}W_L^{(0)} \cdots D_{i,a}^{(0)}W_a^{(0)}}\leq O(\tau \sqrt{m}) \norm{v}$.
\end{lemma}

\begin{lemma}\label{zhu_Thm3}(This Lemma follows Theorem 3 from \cite{allenzhu}) 
Let $\omega=O(\frac{\delta^{3/2}}{n^{9/2}L^6 \log ^3 m})$. With probability at least $1-e^{-\Omega(m\omega ^{2/3}L)}$ over the randomness of $W^{0}, A, B$, it satisfies for every $l\in [L]$ and $W$ with $\norm{W-W^{(0)}}_2\leq\omega$ that
\[ \|\nabla_{W_l} \Phi(W)\|^2_F \leq  O(\tau^2 \Phi(W)\cdot n \cdot m) \]
\end{lemma}

\begin{lemma}\label{Allen-Zhu_7.1}(This Lemma is based on Lemma 7.1 and Lemma 8.2c from \cite{allenzhu})
With high probability over the randomness of $A,W$ we have 
\[ \forall i \in [n], l\in \{0,1,..,L\}: \|h_{i,l}\| = O(1) \]
\end{lemma}

\begin{lemma}\label{init_bound-lemma}
Let $\delta > 0$ and $m\geq \Omega(L \log(nL/\delta )$ then with probability at least  $1-\delta$ it holds that $||u(0)||\leq \sqrt{n}\tau/\delta$ and as a consequence by using the triangle inequality $\Phi(W(0))= \frac{1}{2}\norm{\y-\uu(0)}^2 \leq O(n)$ 
\end{lemma}
\begin{proof}
Conditioned on $W,A$ it holds that  $u_i(0)\backsim N(0,\tau^2 \norm{h_{i,L}}^2)$  and since by Lemma \ref{Allen-Zhu_7.1} we have that $\norm{h_{i,L}} = O(1)$, this yields $E(\norm{\uu(0)}^2)=O\left(n\tau^2\right)$. Then by Markov's inequality,  $ \norm{\uu(0)}^2\leq n\tau^2/\delta^2 $ with probability $1-\delta$.
\end{proof}

\begin{lemma}\label{Arora_thm3.1}(Based on Theorem 3.1 \cite{arora})\footnote{The formulation given in \cite{arora} considers training w.r.t all layers. The proof can be extended trivially to the case where the first and last layers are held fixed. }
Fix $\epsilon>0$ and $\delta \in(0,1)$ and assume  $m\geq \Omega(\frac{L^6}{\epsilon^4}log(\frac{L}{\delta}))$. Then for any pair of inputs $\x_i,\x_j$ such that $\|\x_i\|\leq 1, \|\x_j\|\leq 1$ with probability $1-\delta$ we have
\[ \abs{\frac{1}{m}H_{ij}(0) -\frac{1}{m}H^\infty_{ij}}\leq (L+1)\epsilon \]
\end{lemma}
\begin{lemma}\label{Arora_lemmaF2}(Based on Theorem 5c  \cite{allenzhu})
Let $W^{(0)},A,B$ be at random initialization. For any pair of inputs $\x_i, \x_j$ and parameter $\omega\leq O(\frac{1}{L^9log^{3/2}m})$ with probability at least $1-e^{-\Omega(m\omega^{2/3}L)}$ over $W^{(0)},A,B$ with $\norm{W^{(0)}-W^{(t)}}_2\leq \omega$ we have 
\begin{align}\label{eq:H_elements}
\abs{H_{ij}(t)-H_{ij}(0)} \leq
O(\sqrt{\log m}\cdot \omega^{1/3}L^3)\sqrt{H_{i,i}(0)H_{j,j}(0)}
\end{align}
\end{lemma}

\begin{lemma}\label{corollary_Hinf}
Let $\hat{\delta}\in(0,1] $ and $W^{(0)},A,B$ be at random initialization. Then, for  $m \geq \Omega \left( \frac{n^{24} L^{12} \log^5 m}{\delta^8 \tau^6 } \right)$ and parameter  $\omega = O \left( \frac{n^3}{\delta \tau \sqrt{m}}\log m \right)$ with probability of at least  $1-\hat{\delta}$  over $W^{(0)},A,B$ with  $\norm{W^{(0)}-W^{(t)}}_2\leq \omega$ it holds that
\begin{enumerate}
\item $\norm{H(t)-H(0)}_2\leq O(\frac{n^3log^{5/6} m}{\delta \tau}) m^{5/6}$ 
\item $\norm{H(0)-H^\infty}_2\leq O(\frac{\delta^2 m\tau^3}{n^6})$
\item $\norm{H^\infty-H(t)}_2\leq O(\frac{n^3log^{5/6} m}{\delta \tau}) m^{5/6}+O(\frac{\delta^2 m\tau^3}{n^6})\leq O(\frac{\delta^2 m\tau^3}{n^6}) $
\end{enumerate}
\end{lemma}
\begin{proof}
We prove the first claim. Then, the second claim is obtained by plugging $m$ into Lemma \ref{Arora_thm3.1}. The third claim is a direct consequence of the two claims using triangle inequality.

By the definition of $H_{ij}(0)$ we have that
\begin{align*}
    \sqrt{H_{ii}(0)}&=\sqrt{\left\langle \frac{\partial u_i(0)}{\partial W},\frac{\partial u_i(0)}{\partial W} \right\rangle}\\
    &\leq \sum_{l=1}^{L} \norm{\frac{\partial u_i(0)}{\partial W_{l}}} =\sum_{l=1}^{L} \norm{ h_{i,l-1} BD^{(0)}_{i,L}W^{(0)}_L  D^{(0)}_{i,L-1}W^{(0)}_{L-1}\cdots D^{(0)}_{i,L+1} W^{(0)}_{l+1}D^{(0)}_{i,l}} \\
    &\leq \sum_{l=1}^{L} \norm{ h_{i,l-1}}\norm{ BD^{(0)}_{i,L}W^{(0)}_L  D^{(0)}_{i,L-1}W^{(0)}_{L-1}\cdots D^{(0)}_{i,L+1} W^{(0)}_{l+1}D^{(0)}_{i,l}}\leq O(L\sqrt{m}\tau )
\end{align*}
where the last inequality is obtained by applying Lemma \ref{zhu_7.4} and Lemma \ref{Allen-Zhu_7.1}.
Applying the obtained bound  for $H_{ii}(0)$ and $H_{jj}(0)$ yields a  bound for  $\abs{H_{ij}(t)-H_{ij}(0)}$, using \eqref{eq:H_elements}. Finally, $\norm{H(t)-H(0)}\leq O(\frac{n^3log^{5/6} m}{\delta \tau}) m^{5/6}$.
\end{proof}

\section{Experiment setup}

Below we provide our experimental setup for all the figures in the paper.

\textbf{Figure 1}. 
Experiments are run with input data in $\Sphere^1$ drawn from a uniform (top plots) and non-uniform (bottom plots) distributions, where the latter densities are of ratio $1:40$. The target function is $y(x) = 0.4\cos(16x) + \cos(x)$. The number of training points is $n=10000$ and batch size is 100. The network includes $L=10$ fully connected layers, each with $m=256$ hidden units. The weights are initialized with normal distribution with standard deviation $\tau=0.1$, and the learning rate is $\eta = 0.001$.


\textbf{Figure 2}. Eigenfunctions are computed with $n=2,933$ data points in $\Sphere^1$.

\textbf{Figure 3}. Local frequencies are computed with $n=1,467$ data points in $\Sphere^1$.

\textbf{Figure 4}. Eigenvalues are computed with $n=50,000$ data points in $\Sphere^1$.

\textbf{Figure 5}. Eigenvalues are computed with $n=12,567$ data points in $\Sphere^1$.

\textbf{Figure 6}.
Eigenvectors are computed numerically using $n=10,000$ data points in $\Sphere^1$ drawn from a piecewise constant distribution with densities proportional to $(11,1,3)$.

\textbf{Figure 7}.
Convergence times are measured by training a two-layer network with bias. The weights of the second layer are set randomly to $-1$ or $1$ (with probability $0.5$) and remain fixed throughout training. The bias is initialized to zero. The network parameters are set to $m=4000$, $\eta=0.004$, $n=734$, and $\tau=0.2$. Convergence for region $R_j$ is declared when $\frac{1}{2\abs{R_j}}\sum_{i\in R_j}^n \left(f(x_i;w)-u_i\right)^2<\frac{\delta}{n}$ with $\delta=0.05$.

\textbf{Figure 8}.
Eigenvectors are computed with $n=9,926$ data points in $\Sphere^2$.

\textbf{Figure 9}.
We used the same setup as in Figure 7 with the parameters: $m=8000$, $tau=0.2$, and $\eta=0.004$. Here $n$ varies between the three plots. We sampled 300 points from a uniform distribution on one hemisphere, and $300p_2/p_1$ points on the other hemisphere, where $p_2/p_1 \in \{2,3,4\}$.

\textbf{Figure 10}.
Eigenvectors are computed with $n=1257$ data points in $\Sphere^1$.

\textbf{Figure 11}. Here we compare the number of iterations needed for a deep FC network to converge the number of iterations predicted by the eigenvalue of the corresponding NTK. We used $m=256$,  $\eta=0.05$ and $\delta=0.05$. The corresponding NTK was calculated in the $\Sphere^1$ with $n=630$ points and in $\Sphere^2$ with $n=1,000$ points, both drawn from a uniform distribution. Note that the plot for $\Sphere^2$ appears on the left and the one for $\Sphere^1$ on the right.

\textbf{Figure 12}. Here, we calculate the eigenvalues of NTK for FC networks with $3 \le L \le 50$ layers for data distributed uniformly in $\Sphere^1$ (left) and $\Sphere^2$ (right). 
The NTK was calculated with $n=16,383$ and $n=20,000$ data points in $\Sphere^1$ and $\Sphere^2$, respectively.

\end{document}